\documentclass{article}

\usepackage{microtype}
\usepackage{graphicx}
\usepackage{subfigure}
\usepackage{booktabs} 

\usepackage{hyperref}

\usepackage[utf8]{inputenc} 
\usepackage[T1]{fontenc}    
\usepackage{url}            
\usepackage{booktabs}       
\usepackage{amsfonts}       
\usepackage{nicefrac}       
\usepackage{microtype}      

\usepackage{url}
\usepackage{graphicx}
\usepackage{amsmath}

\usepackage{amsfonts}
\usepackage{graphicx} 
\usepackage{algorithm,algorithmic}
\usepackage{tabularx} 
\usepackage{amsthm}
\usepackage{amsfonts}
\usepackage{amsmath}
\usepackage{amssymb}
\usepackage{mathrsfs}
\usepackage{multirow}
\usepackage{bbm}
\usepackage{wrapfig}

\newtheorem{theorem}{Theorem}
\newtheorem{definition}{Definition}

\newtheorem{lemma}{Lemma}

\usepackage{thmtools,thm-restate}

\DeclareMathOperator*{\argmin}{arg\,min}

\DeclareMathOperator*{\minimize}{minimize}
\DeclareMathOperator*{\maximize}{maximize}
\DeclareMathOperator*{\unif}{Unif}

\DeclareMathOperator*{\rank}{rank}

\usepackage{enumitem}
\usepackage{wrapfig}

\newcommand{\beq}{\begin{equation}}

\newcommand{\eeq}{\end{equation}}
\newcommand{\beqs}{\begin{equation*}}
\newcommand{\eeqs}{\end{equation*}}

\newcommand{\st}{\text{ s.t. }}

\newcommand{\emp}{\textsf{EMP}}
\newcommand{\smp}{\textsf{SMP}}
\newcommand{\proj}{\textsf{Proj}}
\newcommand{\update}{\textsf{Update}}

\newcommand{\aemp}{\textsf{Accel-EMP}}
\newcommand{\asmp}{\textsf{Accel-SMP}}

\renewcommand{\L}{\mathcal{L}}

\newcommand{\LL}{\mathbb{L}}

\newcommand{\R}{\mathbb{R}}

\newcommand{\U}{\mathcal{U}}

\newcommand{\M}{\mathcal{M}}

\newcommand{\E}{\mathbb{E}}

\newcommand{\1}{\mathbf{1}}
\newcommand{\<}{\langle}
\renewcommand{\>}{\rangle}

\usepackage{setspace}
\usepackage{bm}

\newcommand{\y}{\mathbf{y}}
\renewcommand{\v}{\mathbf{v}}
\newcommand{\lambdad}{\bm \lambda}
\newcommand{\gam}{\mu}

\renewcommand{\Pr}{\mathbb{P}}

\renewcommand{\tilde}[1]{\widetilde{#1}}

\usepackage{natbib}

\usepackage[accepted]{icml2020}

\icmltitlerunning{Accelerated Message Passing for Entropy-Regularized MAP Inference}

\begin{document}

\twocolumn[
\icmltitle{Accelerated Message Passing for Entropy-Regularized MAP Inference}



\icmlsetsymbol{equal}{*}

\begin{icmlauthorlist}
\icmlauthor{Jonathan N. Lee}{stan}
\icmlauthor{Aldo Pacchiano}{eecs}
\icmlauthor{Peter Bartlett}{eecs,stat}
\icmlauthor{Michael I. Jordan}{eecs,stat}
\end{icmlauthorlist}

\icmlaffiliation{stan}{Department of Computer Science, Stanford University, USA}
\icmlaffiliation{eecs}{Department of Electrical Engineering and Computer Sciences, University of California, Berkeley, USA}
\icmlaffiliation{stat}{Department of Statistics, University of California, Berkeley, USA}

\icmlcorrespondingauthor{Jonathan Lee}{jnl@stanford.edu}
\icmlcorrespondingauthor{Aldo Pacchiano}{pacchiano@berkeley.edu}

\icmlkeywords{Machine Learning, ICML}

\vskip 0.3in
]



\printAffiliationsAndNotice{}  

\begin{abstract}
   
Maximum a posteriori (MAP) inference in discrete-valued Markov random fields is a fundamental problem in machine learning that involves identifying the most likely configuration of random variables given a distribution.  Due to the difficulty of this combinatorial problem, linear programming (LP) relaxations are commonly used to derive specialized message passing algorithms that are often interpreted as coordinate descent on the dual LP. To achieve more desirable computational properties, a number of methods regularize the LP with an entropy term, leading to a class of smooth message passing algorithms with convergence guarantees.
In this paper, we present randomized methods for accelerating these algorithms by leveraging techniques that underlie classical accelerated gradient methods. The proposed algorithms incorporate the familiar steps of standard smooth message passing algorithms, which can be viewed as coordinate minimization steps. 
We show that these accelerated variants achieve faster rates for finding $\epsilon$-optimal points of the unregularized problem, and,
when the LP is tight, we prove that the proposed algorithms recover the true MAP solution in fewer iterations than standard message passing algorithms.

\end{abstract}

\section{Introduction}
\label{section::introduction}

Discrete undirected graphical models are extensively used in machine learning since they provide a versatile and powerful way of modeling dependencies between variables~\citep{wainwright2008graphical}. In this work we focus on the important class of discrete-valued pairwise models. Efficient inference in these models has multiple applications, ranging from computer vision~\citep{jegelka2011submodularity}, to statistical physics~\citep{mezard2009information}, information theory~\citep{mackay2003information} and even genome research~\citep{torada2019imagene}.

In this paper we study and propose efficient methods for maximum a posteriori (MAP) inference in pairwise, discrete-valued Markov random fields. The MAP problem corresponds to finding a configuration of all variables achieving a maximal probability and is a key problem that arises when using these undirected graphical models. There exists a vast literature on MAP inference spanning multiple communities, where it is known as constraint satisfaction \citep{schiex1995valued} and energy minimization \citep{kappes2013comparative}. Even in the binary case, the MAP problem is known to be NP-hard to compute exactly or even to approximate  \citep{kolmogorov2004energy,cooper1990computational}. 

As a result, there has been much emphasis on devising methods that may work on settings under which the problem becomes tractable. A popular way to achieve this goal is to express the problem as an integer program and then relax this to a linear program (LP). If the LP constraints are set to the convex hull of marginals corresponding to all global settings, also known as the marginal polytope~\citep{wainwright2008graphical}, then the LP would yield the optimal integral solution to the MAP problem. Unfortunately, writing down this polytope would require exponentially many constraints and therefore it is not tractable. We can consider larger polytopes defined over subsets of the constraints required to define the marginal polytope. This is a popular approach that underpins the family of LP relaxations known as the Sherali-Adams (SA) hierarchy \cite{sherali1990hierarchy}. 
Instead of enforcing  global consistency, we enforce only pairwise consistency via the local polytope, thus yielding pseudo-marginals that are pairwise consistent but may not correspond to any true global distribution. Despite the local polytope requiring a number of constraints that is linear in the number of edges of the input graph, the runtime required for solving this linear program for large graphs may be prohibitive in practice \citep{yanover2006linear}. These limitations have motivated the  design and theoretical analysis of message passing algorithms that exploit the structure of the problem. In this paper we study a class of smooth message passing algorithms, derived from a regularized version of the local polytope LP relaxation \citep{ravikumar2010message,meshi2012convergence, savchynskyy2011study,hazan2012convergent}. 

The technique of using entropy penalties to regularize linear programs has a long and successful history. It has been observed, both practically and in theory that, in some problems, solving a regularized linear program yields algorithms with computational characteristics that make them  preferable to simply using an LP solver, particularly with large scale problems. Previous work has studied and analyzed convergence rates, and even rounding guarantees for simple message passing algorithms \cite{ravikumar2010message,lee2019approximate,meshi2012convergence} based on iterative Bregman projections onto the constraints. 
These algorithms are sometimes described as being \textit{smooth}, as the dual objective is smooth in the dual variable as a result of the entropy regularization.
Inspired by accelerated methods in optimization \cite{lee2013efficient,nesterov2012efficiency} we propose and analyze two new variants of accelerated message passing algorithms \aemp{}  and \asmp{}. In this paper, we are able to show our methods drastically improve upon the convergence rate of previous message passing algorithms.

\subsection{Related Work}
\paragraph{MAP Inference}
	The design and analysis of convergent message passing algorithms has attracted a great deal of attention over the years. Direct methods of deriving asymptotically convergent algorithms have been extensively explored. Examples include tree-reweighted message passing \citep{kolmogorov2006convergent}, max-sum diffusion \citep{werner2007linear}, MP-LP \citep{globerson2008fixing}, and other general block-coordinate ascent methods \citep{kappes2013comparative,sontag2011introduction}. Our work builds upon regularized inference problems that directly regularize the linear objective with strongly convex terms, often leading to "smooth" message passing \citep{savchynskyy2011study,savchynskyy2012efficient,hazan2012convergent,weiss2012map,meshi2015smooth}. This formalization has led to a number of computationally fast algorithms, but often with asymptotic guarantees. 
	
	The focus of this paper is non-asymptotic convergence guarantees for families of these algorithms.  \citet{ravikumar2010message} provided one of the first directions towards this goal leveraging results for proximal updates, but ultimately the rates were not made explicit due to the approximation at every update. \citet{meshi2012convergence} provided a comprehensive analysis of a message passing algorithm derived from the entropy-regularized objective and later a gradient-based one for the quadratically regularized objective \citep{meshi2015smooth}. However, the convergence rates were only given for the regularized objective, leaving the guarantee on the unregularized problem unknown. Furthermore, the objective studied by \citet{meshi2015smooth} did not ultimately yield  a message passing (or coordinate minimization) algorithm, which could be more desirable from a computational standpoint. \citet{lee2019approximate} provided rounding guarantees for a related message passing scheme derived from the entropy regularized objective, but did not consider convergence on the unregularized problem either.
	\citet{savchynskyy2011study} studied the direct application of the acceleration methods of \citet{nesterov2018lectures}; however, this method also forwent the message passing scheme and convergence on the unregularized problem was only shown asymptotically. \citet{jojic2010accelerated} gave similar results for a dual decomposition method that individually smooths subproblems. Acceleration was applied to get fast convergence but on the dual problem.
	
		In addition to problem-specific message passing algorithms, there are numerous general purpose solvers that can be applied to MAP inference to solve the LP with strong theoretical guarantees.  Notably, interior point methods \citep{karmarkar1984new,renegar1988polynomial,gondzio2012interior} offer a promising alternative for faster algorithms. For example recent work provides a $\widetilde O(\sqrt{\rank})$ iteration complexity by \citet{lee2014path}, where $\rank$ is the rank of the constraint matrix. 
In this paper, we only consider comparisons between message passing algorithms; however, it would be interesting to compare both empirical and theoretical differences between message passing and interior point methods in the future.

    \paragraph{Accelerating Entropy-Regularized Linear Programs}
    We also highlight a connection with similar problems in other fields.
    Notably, optimal transport also admits an LP form and has seen a surge of interest recently. As in MAP inference, these approximations are conducive to empirically fast algorithms, such as the celebrated \textit{Sinkhorn} algorithm, that outperforms generic solvers \citep{cuturi2013sinkhorn, benamou2015iterative,genevay2016stochastic}. In theory, \citet{altschuler2017near} showed convergence guarantees for Sinkhorn and noted that it can be viewed as a block-coordinate descent algorithm on the dual, similar to the MAP problem. Since this work, several methods have striven to obtain faster rates \citep{lin2019acceleration,dvurechensky2018computational}, which can be viewed as building on the seminal acceleration results of \citet{nesterov2018lectures} for general convex functions.  
    It is interesting to note that the entropy-regularized objectives in optimal transport and MAP inference effectively become softmax minimization problems, which have also been studied generally in the context of smooth approximations \citep{nesterov2005smooth} and maximum flow  \citep{sidford2018coordinate}.

\subsection{Contributions}

For the case of MAP inference from entropy-regularized objectives, we address the question: is it possible to directly accelerate message passing algorithms with faster non-asymptotic convergence and improved rounding guarantees? We answer this question affirmatively from a theoretical standpoint. We propose a method to directly accelerate standard message passing schemes, inspired by Nesterov.
We prove a convergence guarantee for standard schemes on the unregularized MAP objective over $\LL_2$, showing convergence on the order of $\widetilde O(m^5/\epsilon^3)$ iterations where $m$ is the number of edges, assuming the number of vertices and labels and the potential functions are fixed. We then prove that the accelerated variants converge in expectation on the order of $\widetilde O(m^{9/2}/\epsilon^2)$ iterations. We conclude by showing that the accelerated variants recover the true MAP solution with high probability in fewer iterations compared to prior message passing analyses when the LP relaxation is tight and the solution is unique \citep{lee2019approximate}.

%

\subsection{Notation} Let $\R_+$ denote the set of non-negative reals. The $d$-dimensional probability simplex over the finite set $\chi$ is $\Sigma^d := \left\{ p \in \R^d_+  \ : \ \sum_{x \in \chi} p(x) = 1 \right\}$. A joint distribution, $P \in \Sigma^{d\times d}$, is indexed by $x_c = (x_p, x_q) \in \chi^2$.
The transportation polytope of $p, q\in \Sigma^d$ is defined as the set of pairwise joint distributions that marginalize to $p$ and $q$, written as $\U_d(p, q) = \{ P \in \R^{d\times d}_+ \ : \ \sum_{x_p} P(x_p, x) = q(x), \ \sum_{x_q} P(x, x_q) = p(x) \}$. For any vector $p \in \R^d_+$, we write the entropy as $H(p) := - \sum_{x} p(x) (\log p(x) - 1)$. While this is a somewhat unusual definition, it simplifies terms later and has been used by \cite{benamou2015iterative,lee2019approximate}. We will use $\<\cdot, \cdot\>$ generally to mean the sum of the elementwise products between two equal-length indexable vectors.
 For any two vectors $p, q \in \R^d_+$, the Hellinger distance is  $h(p, q): = \frac{1}{\sqrt 2}\|\sqrt p - \sqrt q\|_2$. The vector $\1_d \in \R^d$ consists of all ones.



\section{Coordinate Methods for MAP Inference}\label{sec::setup}

\subsection{Smooth Approximations}
For the pairwise undirected graph $G = (V, E)$ with $ n := |V| $ and $m:= |E|$\footnote{In this paper we study pairwise models with only vertices and edges, implying only pairwise interaction between variables. However, our results can be extended to more general graphs.}, we associate each vertex $i \in V$ with the random variable $X_i$ on the finite set of labels $\chi$ of size $d= |\chi| \geq 2$ and a distribution that factorizes as $p(x_V) = \frac{1}{Z(\phi) } \prod_{e \in E} \phi_e(x_e) \prod_{i \in V} \phi_i(x_i)$ where $\phi_i \in \mathbb{R}^{d}$ and $\phi_e \in \mathbb{R}^{d^2}$. For any edge $e \in E$, we use $i \in e$ to denote that $i$ is one of the two endpoints of $e$.
For $i \in V$, we define $N_i := \{ e \in E \ : \ i \in e\}$, as the set of all incident edges to $i$. We assume that each vertex has at least one edge. MAP inference refers to the combinatorial problem of identifying the configuration that maximizes this probability distribution. In our case, we cast it as the following minimization problem:
\begin{align}\label{eq::map}\tag{MAP}
\minimize_{x_V \in \chi^{|V|}} \quad \textstyle \sum_{i \in V} C_i(x_i) + \textstyle \sum_{e \in E} C_e(x_e),
\end{align}
where $C = - \log \phi$, i.e., we view $C$ as indexable by vertices, edges, and their corresponding labels. 
It can be shown \cite{wainwright2008graphical} that (\ref{eq::map}) is equivalent to the following linear program:
\begin{align*}
\min_{\mu \in \M} \ \< C, \mu \> \quad \st \quad \mu \in \M,
\end{align*}
where $\mu \in \R^{r_P}$ for $r_P = nd + md^2$ is known as a marginal vector, and $\<C, \mu\> = \sum_{i \in V} \sum_{x_i \in \chi} C_i(x_i) \mu_i(x_i) + \sum_{e \in E} \sum_{x_e \in \chi^2} C_e(x_e) \mu_e(x_e)$, and $\M$ is the marginal polytope defined as
\begin{align*}
    \M := \left\{ \mu  \ : \ \exists \ \Pr \ \st\begin{array}{lr}
	\mathbb P_{X_i}(x_i) = \mu_i(x_i), \ \forall i, x_i \\
	\mathbb P_{X_i X_j}(x_e) = \mu_{e}(e), \  \forall e, x_e   
	\end{array} \right\}.
\end{align*}
Here, $\Pr$ is any valid distribution over the random variables $\{ X_i\}_{i \in V}$. Since $\M$ is described by exponentially many constraints in the graph size, outer-polytope relaxations are a standard paradigm to approximate the above problem by searching instead over the local polytope:
\begin{align*} 
	\LL_2 := \left\{ \mu  \ : \ \begin{array}{lr}
	\mu_i \in \Sigma^d & \forall i \in V \\
	\mu_{e} \in \U_d(\mu_i, \mu_j)  & \forall e = ij \in E\\    
	\end{array} \right\}.
\end{align*}
The local polytope $\LL_2$ enforces only pairwise consistency between variables while $\M$ requires the marginal vector to be generated from a globally consistent distribution of $\{X_i\}_{i \in V}$, so that $\M\subseteq \LL_2$. We refer the reader to the survey of \citet[\S3]{wainwright2008graphical} for details. Thus, our primary objective throughout the paper will be finding solutions to the approximate problem
\begin{align*}\label{eq::original}\tag{P}
    \minimize \quad \< C, \mu \> \quad \st \quad \mu \in \LL_2.
\end{align*}
Let $\epsilon > 0$. We say that a point $\widehat \mu \in \LL_2$ is $\epsilon$-optimal for (\ref{eq::original}) if it satisfies $\< C, \widehat \mu \> \leq \min_{\mu \in \LL_2} \< C, \mu\> + \epsilon$. 
For a random $\widehat \mu$, we say that it is expected $\epsilon$-optimal if
\begin{align*}
    \E \left[ \< C, \widehat \mu \> \right] \leq \min_{\mu \in \LL_2} \<C, \mu\> + \epsilon.
\end{align*}
Despite the simple form of the linear program, it has been observed to be difficult to solve in practice for large graphs even with state-of-the-art solvers \cite{yanover2006linear}, motivating researchers to study an approximate version with entropy regularization:
\begin{align}\label{eq::primal}\tag{Reg-P}
\minimize \quad \< C, \mu\> - \frac{1}{\eta} H(\mu) \quad \st \quad \mu \in \LL_2, 
\end{align}
where $\eta \in \R_+$ controls the level of regularization. Intuitively, the regularization encourages $\mu_i$ and $\mu_e$ to be closer to the uniform distribution for all vertices and edges.

The dual problem takes on the succinct form of an unconstrained log-sum-exp optimization problem. Thus, when combined, the local polytope relaxation and entropy-regularizer result in a smooth approximation.
\begin{restatable}{proposition}{propdualobj}
\label{prop::dual-obj}
	The dual objective of (\ref{eq::primal}) can be written as
	\begin{align}\label{eq::dual}\tag{Reg-D}
	\minimize_{\lambdad} \quad L(\lambdad),
	\end{align}
	where $L$ is defined as
	\begin{align*}
	L(\lambdad) & = \frac{1}{\eta} \sum_{i \in V} \log  \sum_{x \in \chi} \exp\left( - \eta C_i(x) + \textstyle \sum_{e \in N_i} \lambdad_{e, i}(x) \right) \\
	& + \frac{1}{\eta} \sum_{e \in E}  \log  \sum_{x \in \chi^2} \exp\left( - \eta C_e(x) - \textstyle \sum_{i \in e} \lambdad_{e, i}(x_i) \right).
	\end{align*}
	Furthermore, primal variables can be recovered directly by
	\begin{align*}
	\mu_i^{\lambdad}(x_i) & \propto \exp \left( - \eta C_i(x_i) + \eta \textstyle \sum_{e \in N_i} \lambdad_{e, i}(x_i) \right)  \\
	\mu_{e}^{\lambdad} (x_e) & \propto \exp \left( - \eta C_e(x_e) - \eta \textstyle \sum_{i \in e} \lambdad_{e, i}((x_e)_i)\right).
	\end{align*}
\end{restatable}
For convenience we let $r_D = 2md$ denote the dimension of the dual variables $\lambdad \in \R^{r_D}$. This is in contrast to the dimension $r_P$ of the primal marginal vectors defined earlier. We use $\Lambda^* \subseteq \R^{r_D}$ to denote the set of solutions to (\ref{eq::dual}).

There is a simple interpretation to dual optimality coming directly from the Lagrangian: a dual variable $\lambdad$ is optimal if the candidate primal solution is primal feasible: $\mu^{\lambdad} \in \LL_2$. It can be seen that the derivative of the dual function $L(\lambdad)$ captures the slack of a $\mu^{\lambdad}$:
\begin{align}\label{eq::derivative}
    \frac{\partial L(\lambdad) }{\partial \lambdad_{e, i}(x_i)} = \mu_{i}^{\lambdad}(x_i) - S_{e, i}^{\lambdad}(x_i)
\end{align}
where we define $S_{e, i}^{\lambdad}(x_i) := \sum_{x_j \in \chi} \mu_{e}^{\lambdad}(x_i, x_j)$. The gradient captures the amount and direction of constraint violation in $\LL_2$ by $\mu^{\lambdad}$. In order to discuss this object concisely and intuitively, we formally define the notion of a slack vector, which is simply the negative of the gradient, and a slack polytope \citep{lee2019approximate}, which describes the same polytope as $\LL_2$ if the constraints were offset by exactly the amount by which $\mu^{\lambdad}$ is offest.
\begin{definition}[Slack vector and slack polytope]\label{def::slack}
 For $\lambdad \in \R^{r_D}$, the slack vector $\nu^{\lambdad} \in \R^{r_D}$ of $\lambdad$ is defined as $\nu^{\lambdad}_{e, i}(x) = S_{e, i}^{\lambdad} (x) - \mu_i^{\lambdad}(x)$ for all $e \in E$, $i \in e$, and $x \in \chi$.
 
 The slack polytope for a slack vector $\nu$ is defined as
 \begin{align*}
	\LL_2^{\nu} := \left\{ \mu \in \R^{r_P}_+  \ : \ \begin{array}{lr}
	\mu_i \in \Sigma^d  \\
	\mu_{e} \in \U_d(\mu_i + \nu_{e, i} , \mu_j +\nu_{e, j} ) 
	\end{array}\right\}
\end{align*}
\end{definition}

\subsection{Entropy-Regularized Message Passing}
The results in this paper will be primarily concerned with algorithms that approximately solve (\ref{eq::map}) by directly solving (\ref{eq::dual}). For solving this objective, message passing algorithms can effectively be viewed as block-coordinate descent, except that a full minimization is typically taken at each step. Here we outline two variants.

\subsubsection{Edge Message Passing} Edge message passing (EMP) algorithms reduce to block-coordinate methods that minimize (\ref{eq::dual}) for a specific edge $e = \{ i, j\} \in E$ and endpoint vertex $i \in e$, while keeping all other dual variables fixed. Let $L_{e, i}(\cdot ; \lambdad)~:~\R^d \to \R$ denote the block-coordinate loss of $L$ fixed at $\lambdad$ except for free variables $\{\lambdad_{e, i}(x)\}_{x \in \chi}$. For each $i \in V$, $e \in N_i$ we define the \emp{} operator for $\lambdad \in \R^{r_D}$:
	\begin{align*}
	{\emp}_{e, i}^\eta (\lambdad) \in \argmin_{\lambdad'_{e, i} \in \R^d}  L_{e, i}(\lambdad'_{e, i}(\cdot); \lambdad).
	\end{align*}
\begin{restatable}{proposition}{propempupdate}
\label{prop::emp-update}
The operator $\emph\emp{}_{e, i}^\eta: \lambdad  \mapsto \lambdad'_{e, i}(\cdot) \in \R^d$ is satisfied by
$\lambdad'_{e, i}(x_i) = \lambdad_{e, i}(x_i) + \frac{1}{2\eta}\log \frac{S_{e, i}^{\lambdad}(x_i)}{\mu_i^{\lambdad} (x_i)}$.
\end{restatable}

In the entropy-regularized setting, this update rule has been studied by \citet{lee2019approximate,ravikumar2010message}. Non-regularized versions based on max-sum diffusion have much earlier roots and also been studied by \citet{werner2007linear,werner2009revisiting}; however, we do not consider these particular unregularized variants.
\emp{} offers the following improvement on $L$.
\begin{restatable}{lemma}{lemempimprovement}
\label{lem::emp-improvement}
Let $\lambdad'$ be the result of applying $\emph \emp_{e, i}^\eta( \lambdad)$ to $\lambdad$, keeping all other coordinates fixed. Then, $L(\lambdad) - L(\lambdad') 
 \geq \frac{1}{4\eta} \| \nu_{e, i}^{\lambdad}\|_1^2$.
\end{restatable}

\begin{figure}
	\begin{algorithm}[H]
		\caption{  \textsf{Standard-MP}$(\update, \eta,P,K )$} \label{alg::semp}
		\begin{algorithmic}[1]
			\STATE $\lambdad^{(0)} = 0$
			\FOR {$k = 0, 1, \ldots, K - 1$}
			\STATE Set $\lambdad^{(k +1)} = \lambdad^{(k)}$
			\STATE Sample block-coordinate $b_k \sim P$
			\STATE Set $\lambdad^{(k + 1)}_{b_k} = \update_{b_k}^\eta (\lambdad^{(k)})$
			\ENDFOR
			\STATE\textbf{return} $\argmin_{\lambdad \in \{ \lambdad^{(k)} \} } \sum_{e \in E, i \in e} \| \nu^{\lambdad}_{e, i} \|_1^2$
		\end{algorithmic}
	\end{algorithm}
	\vspace{-.5cm}
\end{figure}

\subsubsection{Star Message Passing} Star message passing (SMP) algorithms consider block-coordinates that include all edges incident to a particular vertex $i \in V$. For $\lambdad \in \R^{r_D}$ and $i \in V$ let $L_{i}(\cdot ; \lambdad)~:~\R^d \to \R$ denote the block-coordinate loss of $L$ fixed at $\lambdad$. For a given $i \in V$ and arbitrary $\lambdad \in \R^{r_D}$, we define the $\smp{}$ operator:
\begin{align*}
    {\smp{}}_{i}^\eta(\lambdad) \in \argmin_{\lambdad'_{\cdot, i} \in \R^{d | N_i| }} L_i(\lambdad'_{\cdot, i} (\cdot); \lambdad)
\end{align*}
That is, \smp{} is the minimization over the block-coordinate at vertex $i$ for all edges incident to $i$ in $N_i$ and all possible labels in $\chi$.
\begin{restatable}{proposition}{propsmpupdate}
\label{prop::smp-update}
The operator $\emph \smp{}_{i}^\eta : \lambdad \mapsto \lambdad'_{\cdot, i} (\cdot) \in \R^{d|N_i|}$ is, for all $e \in N_i$ and $x_i \in \chi$, satisfied by
\begin{align*}
\lambdad'_{e, i}(x_i) & = \lambdad_{e, i} +  \frac{1}{\eta} \log S_{e, i}^{\lambdad}(x_i) \\
&  - \frac{1}{\eta (|N_i| + 1) } \log \left( \mu^{\lambdad}_i(x_i) \textstyle \prod_{e' \in N_i} S_{e', i}^{\lambdad}(x_i)\right),
\end{align*}
\end{restatable}

The proof is similar to the previous one and is deferred to the appendix. \citet{meshi2012convergence} gave a concise definition and analysis of algorithms from this update rule, and similar star-based algorithms have existed much earlier \citep{wainwright2008graphical}, such as MP-LP \cite{globerson2008fixing}. Due to \citet{meshi2012convergence}, \smp{} also has an improvement guarantee.

\begin{restatable}{lemma}{lemsmpimprov}
\label{lem::smp-improvement}
Let $\lambdad'$ be the result of applying $\emph \smp_{i}^\eta$ to $\lambdad$, keeping all other coordinates fixed. Then,
$ L(\lambdad) - L(\lambdad')  
    \geq  \frac{1}{8|N_i|\eta} \sum_{e \in N_i}  \|\nu_{e, i}^{\lambdad}\|_1 ^2
$.
\end{restatable}

\begin{figure}
	\begin{algorithm}[H]
		\caption{  \aemp$(G, C, \eta, K)$} \label{alg::accel-emp}
		\begin{algorithmic}[1]
			\STATE $\lambdad^{(0)} = 0$, $\v^{(0)} = 0$,  $\theta_{-1} = 1$
			\FOR {$k = 0, 1, \ldots, K - 1$}

			\STATE $\theta_k = \frac{-\theta_{k - 1}^2 + \sqrt{\theta_{k - 1}^4 + 4 \theta_{k - 1}^2}}{2}$
			\STATE $\y^{(k)} = \theta_k \v^{(k)} + ( 1 - \theta_k)  \lambdad^{(k)}$
			\STATE Sample $(e_k, i_k)\sim \unif\left\{ 
			(e, i) \ : 
			\ e \in E, i \in e
			\right\}$.
			\STATE Set $\lambdad^{(k+ 1)} = \lambdad^{(k)}$
			\STATE $\lambdad_{e, i}^{(k + 1)}(\cdot) = \emp{}_{e_k, i_k}^\eta (\y^{(k)})$ 
			\STATE $\v^{(k + 1)} = \v^{(k)}$
			\STATE $\v_{e_k, i_k}^{(k + 1)} = \v_{e_k, i_k}^{(k)} + \frac{1}{2 m \eta \theta_k} \nu^{\y^{(k)}}_{e_k, i_k} $ 
			\ENDFOR
			\STATE\textbf{return} $\lambdad^{(K)}$
		\end{algorithmic}
	\end{algorithm}
	\vspace{-.5cm}
\end{figure}

\subsection{Randomized Standard Algorithms}
The message passing updates described in the previous subsection can be applied to each block-coordinate in many different orders. In this paper, we consider using the updates in a randomized manner, adhering to the generalized procedure presented in Algorithm~\ref{alg::semp}. The algorithm takes as input the update rule \update{}, which could be \emp{} or \smp{}, and a regularization parameter $\eta$. It also requires a distribution $P$ over block-coordinates $b_k$ for each iteration $k \leq K$. In this paper, we will use the uniform distribution over edge-vertex pairs for \emp{}:
\begin{align}\label{eq::uniform}
    b_k = (e_k, i_k) \sim \unif (\{ (e, i) \ : \ e \in E,\  i \in e\}).
\end{align} 
For \smp{}, we use a categorical distribution over vertices based on the number of neighbors of each vertex: 
\begin{align}\label{eq::non-uniform}
    b_k = i_k \sim \text{Cat}( V, \{ p_i\}_{i \in V} ),
\end{align}
where $p_i = \frac{|N_i|}{\sum_{j \in V} |N_j|}$ for each $i \in V$.

\begin{figure}
	\begin{algorithm}[H]
		\caption{ \asmp$(G, C, \eta, K)$} \label{alg::accel-smp}
		\begin{algorithmic}[1]
			\STATE $\lambdad^{(0)} = 0$, $\v^{(0)} = 0$,  $\theta_{-1} = 1$
			\FOR {$k = 0, 1, \ldots, K - 1$}
			\STATE $\theta_k = \frac{-\theta_{k - 1}^2 + \sqrt{\theta_{k - 1}^4 + 4 \theta_{k - 1}^2}}{2}$
			\STATE $\y^{(k)} = \theta_k \v^{(k)} + ( 1 - \theta_k)  \lambdad^{(k)}$
			\STATE Sample $i_k \sim \{ p_i \}_{i \in V}$
			\STATE Set $\lambdad^{(k+ 1)} = \lambdad^{(k)}$
			\STATE $\lambdad_{\cdot, i}^{(k + 1)}(\cdot) = \smp{}_{ i_k}^\eta (\y^{(k)})$
			\STATE $\v^{(k + 1)} = \v^{(k)}$
            \FOR{ $e \in N_{i_k}$ }
			    \STATE $\v_{e, i_k}^{(k + 1)} = \v_{e, i_k}^{(k)} + \frac{\min_j |N_j| }{2p_{i_k}\theta_k\eta N} \nu^{\y^{(k)}}_{e, i_k} $ 
			\ENDFOR
            \ENDFOR
			\STATE\textbf{return} $\lambdad^{(K)}$
		\end{algorithmic}
	\end{algorithm}
	\vspace{-.5cm}
\end{figure}

\section{Accelerating Entropy-Regularized Message Passing}
We now present a scheme for accelerating message passing algorithms in the entropy-regularized formulation. The key idea is to leverage the block-coordinate nature of standard message passing algorithms. We draw inspiration from both the seminal work of \citet{nesterov2018lectures} on accelerating gradient methods and accelerating randomized coordinate gradient methods similar to \citet{lee2013efficient,lu2015complexity}; however, the presented method is specialized to incorporate full block-coordinate minimization at each round, so as to be consistent with existing message passing algorithms used in practice. Furthermore, we can leverage the same simple sampling procedures for selecting the block-coordinates. The presented scheme can thus be viewed as a direct method of acceleration in that standard message passing algorithms can be plugged in.


The scheme is presented in Algorithm \ref{alg::accel-emp} for \emp{} and Algorithm \ref{alg::accel-smp} for \smp{}. 
In \aemp{}, at each round $k$, a random coordinate block is sampled uniformly. That coordinate block for $\lambdad$ is then updated with a step of $\textsf{EMP}^\eta_{e_k, i_k}$ evaluated at $\y^{(k)}$. A block-coordinate gradient step evaluated at $\y^{(k)}$ in the form of the slack vector $\nu_{e_k, i_k}^{\y^{(k)}}$ is also applied to $\v^{(k)}$. 
\asmp{} works similarly but we instead sample from the non-uniform distribution defined in (\ref{eq::non-uniform}). The choice of distributions ultimately determines the step size for $\v$.


As in the case of the standard smooth message passing algorithms, which optimize the dual function (\ref{eq::dual}), the returned solutions may not be primal feasible in finite iterations. To obtain feasible solutions, we consider a projection operation \proj{}, shown in Algorithm~\ref{alg::proj}, that is simply a repeated application of Algorithm~2 of \citet{altschuler2017near}, originally designed for optimal transport. The method effectively finds an edge marginal $\widehat \mu_e$ that sums to the given vertex marginals $\mu_i$ and $\mu_j$ for $e = (i, j)$ plus some optional slack $\nu$. For all practical purposes, we would always set the slack to be $\nu = 0$, ensuring that \proj{} outputs a point in $\LL_2$; however, it will become useful to project points from $\LL_2$ into a particular slack polytope $\LL_2^{\nu^{\lambdad}}$ for the analysis.
When projecting onto $\LL_2$, \proj{} does not require modifying the vertex marginals, so there is no ambiguity if the approximate solution is ultimately rounded to an integral solution using a simple vertex rounding scheme\footnote{Ambiguity could arise for more sophisticated rounding schemes but we do not consider those here.}.

\section{Main Results}\label{sec::main-results}
We now present iteration complexities for the above algorithms for finding $\epsilon$-optimal solutions to the original unregularized problem (\ref{eq::original}) over $\LL_2$. These guarantees make it easy to compare various algorithms as they do not inherently depend on the tightness of the relaxation, rounding heuristics to find integral solutions, or arbitrary choices of $\eta$.

\subsection{Standard Algorithms}
Our first result bounds the number of iterations required to compute $\epsilon$-optimal solutions to (\ref{eq::original}). To recapitulate, prior work by \citet{lee2019approximate} for \emp{} only provided an iteration guarantee for bounding the norm of the slack vector. They also gave guarantees for the number of iterations required to round to the optimal MAP solution when it is available. \citet{meshi2012convergence} gave a guarantee on both the primal and dual \textit{regularized} problems (\ref{eq::primal}) and (\ref{eq::dual}), but not the original (\ref{eq::original}) and without rounding or tuning of $\eta$. Additionally, both works focused mostly on the $\epsilon$ dependence in the rate rather than  actually specifying the graph parameters $m$, $n$, and $d$.

\begin{figure}
	\begin{algorithm}[H]
		\caption{ \proj$(\mu, \nu)$} \label{alg::proj}
		\begin{algorithmic}[1]
            \STATE Set $\widehat \mu_i = \mu_i$ for all $i \in V$
			\FOR {$(i,j) = e \in E$}
			    
			    \STATE Compute $\widehat \mu_e$ by applying Algorithm 2 of \cite{altschuler2017near} on $\mu_e$ with transportation polytope $\U_d(\mu_i + \nu_{e, i}, \mu_j + \nu_{e, j})$
			    
			\ENDFOR
			\STATE\textbf{return} $\widehat \mu$
		\end{algorithmic}
	\end{algorithm}
	\vspace{-.5cm}
\end{figure}

In contrast to these prior works, we give a guarantee on optimality for (\ref{eq::original}) for the standard randomized algorithms, specifying exactly the dependence on graph parameters. The purpose of this extension is to standardize convergence guarantees for the true relaxed problem, which will ultimately be handy for comparing to our primary contribution on the accelerated algorithms.
\begin{theorem}\label{th::standard}
Let $\widehat \lambdad$ be the result of running Algorithm~\ref{alg::semp} with \emph{\emp{}}, uniform sampling distribution (\ref{eq::uniform}) and $\eta~=~\frac{4(m + n) \log d}{\epsilon}$. Let $\widehat \mu = \emph \proj(\mu^{\widehat \lambdad}, 0)$ be its projection onto $\LL_2$.
Then, the number of iterations sufficient for $\widehat \mu$ to be expected $\epsilon$-optimal is
\begin{align*}
 O\left(m d^2 (m + n)^4 \| C\|_\infty^3 \epsilon^{-3} \log d\right).
\end{align*}
If $\widehat \lambdad$ is the output of Algorithm~\ref{alg::semp} using \emph \smp{} and sampling distribution (\ref{eq::non-uniform}), and $\widehat \mu := \emph{\proj}(\mu^{\widehat \lambdad}, 0)$ then $\widehat \mu$ is expected $\epsilon$-optimal in the same order of iterations.

\end{theorem}
The rate harbors a $O(1/\epsilon^3)$ dependence, which at first appears to be worse those of \citet{meshi2012convergence} and \citet{lee2019approximate}; however, their convergence guarantees hold only for the regularized objective. The extra $O(1/\epsilon)$ in our guarantee occurs in the conversion to the original unregularized problem (\ref{eq::original}), which is a stronger result.
It is interesting to observe that the guarantees are effectively the same for both variants, despite having somewhat different analyses. We hypothesize that this is due to the fact that the ``smoothness'' constant for \smp{} in Lemma~\ref{lem::smp-improvement} is greater than that of \emp{} in Lemma~\ref{lem::emp-improvement}. Therefore, the larger block-coordinate size is effectively canceled by the smaller improvement per step.

We now describe the proof briefly here since the first part is fairly standard while the second will be covered in the proof of our main acceleration result. The full proof is found in Appendix~\ref{sec::classic-proof}. The basic idea is to use Lemma~\ref{lem::emp-improvement} to lower bound the expected improvement each iteration, which can be done in terms of the average squared norms of the slack $\|\nu_{e, i}\|_1^2$. We can guarantee improvement on $L$ by at least $(\epsilon')^2$ until the norms are on average below $\epsilon'$.
Knowing that the slack norms are small, we can prove that the projection $\widehat \mu$ of $\mu^{\widehat\lambdad}$ onto $\LL_2$ is not too far from $\mu^{\widehat\lambdad}$ and so the expected value of $\< C, \widehat \mu\>$ is not much worse than that of $\<C, \mu^{\lambdad}\>$. We then prove that $\<C, \mu^{\lambdad}\>$ is small with respect to the slack norms up to some entropy term, and we set $\eta$ so the entropy term is sufficiently small with respect to a given $\epsilon > 0$.

\subsection{Accelerated Algorithms}

Our primary result gives improved iteration complexities for the accelerated versions of \emp{} and \smp{}. 
To do so, we rely on the classic estimate sequence method initially developed by Nesterov. In particular, we turn to a randomized variant, which has appeared before in the literature on randomized coordinate gradient methods by \citet{lee2013efficient,lu2015complexity} for the strongly convex and general cases respectively. Our main contributions are both extending these results for the full minimization of message passing to achieve the fast rates and also proving the accelerating guarantee on the original relaxed problem (\ref{eq::original}) rather than the regularized problems.

\begin{theorem}\label{th::accel}
Let $\widehat \lambdad$ be the output of  Algorithm~\ref{alg::accel-emp} with $\eta~=~\frac{4(m + n) \log d}{\epsilon}$. Let $\hat \mu = \emph{\proj}(\mu^{\hat \lambdad}, 0)$ be its projection onto $\LL_2$. Then, the number of iterations sufficient for $\hat \mu$ to be expected $\epsilon$-optimal is
\begin{align*}
 O\left(m^{3/2} d^2 (m + n)^3 \| C\|_\infty^2 \epsilon^{-2} \log d\right).
\end{align*}
If $\widehat \lambdad$ is the output of Algorithm~\ref{alg::accel-smp} and $\widehat \mu := \emph{\proj}(\mu^{\widehat \lambdad}, 0)$, then $\widehat \mu$ is expected $\epsilon$-optimal in the same order of iterations.
\end{theorem}

The primary difference between the iteration complexities for the standard algorithms and the accelerated ones is the dependence on $\epsilon$. For the accelerated algorithms, we are left with only a $O(1/\epsilon^2)$ dependence versus the $O(1/\epsilon^3)$ dependence for the standard algorithms. This can lead to far fewer iterations to get the same level of accuracy on the original relaxed problem (\ref{eq::original}). In addition, the bounds in Theorem~\ref{th::accel} are strictly better in dependence on the number of edges as well for both \emp{} and \smp{}. For $m + n \approx m$, the accelerated algorithms shave off a $m^{1/2}$ factor. It is interesting to observe that these improved guarantees come with virtually no extra computation per iteration. For example,  to update the sequences $\lambdad^{(k)}$, $\v^{(k)}$, and $\y^{(k)}$ in \emp{} at each iteration, we need only compute the primal variables $\mu_{i_k}^{\lambdad^{(k)}}$ and $S_{e_k, i_k}^{\lambdad^{(k)}}$ once to use in both the slack vector $\nu_{e_k, i_k}^{\lambdad^{(k)}}$ and the update rule $\emp_{e_k, i_k}^{\eta} (\lambdad^{(k)})$.

We will give the proof for \aemp{} to convey the main idea. The analogous \asmp{} case can be found in the appendix. First, we will derive a faster convergence rate on the dual objective, which in turn implies that we can bound the slack norms by the same $\epsilon' > 0$ in fewer iterations. 
In the second part, we will bound the approximation error caused by the entropy regularization. 
Finally, we put these pieces together to determine the appropriate choice of $\epsilon'$ and $\eta$ in terms of $\epsilon$ to recover the final rate.

\vspace{-.3cm}
\subsubsection{Faster Convergence on the Dual}
The first steps will involve defining a randomized estimate sequence to work with and then using this sequence to prove the faster convergence rate on the dual objective.

\begin{definition}
Let $\phi_0: \R^{r_D} \to \R$ be an arbitrary deterministic function. A sequence $\{ \phi_k, \delta_k\}_{k = 0}^K$ of random functions $\phi_k: \R^{r_D} \to \R$ for $k \geq 1$ and deterministic real values $\delta_k \in \R_+$ is a randomized sequence for $L(\lambdad)$ if it satisfies $\delta_k \stackrel{k}{\rightarrow} 0$ and, for all $k$,
$\E[ \phi_{k}(\lambdad) ] \leq (1 - \delta_k) L(\lambdad) + \delta_k \phi_0(\lambdad)$.
\end{definition}
If we are given a random estimate sequence  $\{ \phi_k, \delta_k\}_{k = 0}^K$ and a random sequence $\{\lambdad^{(k)}\}_{k = 0}^K$ that satisfies $\E[ L(\lambdad^{(k)})]  \leq  \min_{\lambdad} \E \left[ \phi_k(\lambdad) \right]$, then
\begin{align}\begin{split}\label{eq::estimate-sequence-bound}
\E [L(\lambdad^{(k)})] - L(\lambdad^*) & \leq \min_{\lambdad} \E [ \phi_k(\lambdad)] - L(\lambdad^*) \\
& \leq \delta_k( \phi_0(\lambdad^*) - L(\lambdad^*)) 
\end{split}
\end{align}
This expected error converges to zero since $\delta_k \stackrel{k}{\rightarrow} 0$.
We now identify a candidate estimate sequence. Let $\lambdad^{(0)} = 0$, $\delta_0 = 1$, $\lambdad^* \in \Lambda^*$, and $q := 2m$. Let the sequence $\{ \theta_k\}_{k = 0}^K$ be as it is defined in Algorithm~\ref{alg::accel-emp} and let $\{ \y^{(k)} \}_{k = 0}^K \subset \R^{r_D}$ be arbitrary. Consider $\{\phi_k, \delta_k\}_{k = 0}^K$ defined recursively as
\begin{align}\begin{split} \label{eq::estimate-sequence}
    (e_k, i_k) & \sim \unif\{ (e, i) : e \in E, i \in e \} \\
    \delta_{k + 1} & = (1 - \theta_k) \delta_k \\
    \phi_{k + 1} (\lambdad) & = (1 - \theta_k) \phi_k(\lambdad) + \theta_k L(\y^{(k)})  \\
    & \quad - q \theta_k \< \nu_{e_k, i_k}^{\y^{(k)}}, \lambdad_{e_k, i_k} - \y^{(k)}_{e_k, i_k}\>
    \end{split}
\end{align}
where $\phi_0(\lambdad) = L(\lambdad^{(0)}) + \frac{\gamma_0}{2}\| \lambdad^{(0)} - \lambdad\|^2_2$ for $\gamma_0 = 2q^2 \eta$.

\begin{restatable}{lemma}{lemestseq}
\label{lem::est-seq}
The sequence $\{\phi_k, \delta_k\}_{k = 0}^K$ defined in (\ref{eq::estimate-sequence}) is a random estimate sequence. Furthermore, it maintains the form $\phi_k(\lambdad) = \omega_k + \frac{\gamma_k}{2}\| \lambdad - \v^{(k)} \|$ for all $k$ where
 \begin{align*}
     \gamma_{k + 1} & = (1 - \theta_k) \gamma_k \\
     \v^{(k + 1)}_{e, i} & = \begin{cases} 
     \v_{e, i}^{(k)} + \frac{q\theta_k}{\gamma_{k + 1}} \nu_{e, i}^{\y^{(k)}} & \text{if }  (e, i) = (e_k, i_k) \\
     \v_{e, i}^{(k)} &  \emph{\text{otherwise}}
     \end{cases}\\
     \omega_{k + 1} & = (1 - \theta_k) \omega_k + \theta_k L(\y^{(k)})  - \frac{(\theta_kq)^2}{2\gamma_{k+1}} \| \nu_{e_k, i_k}^{\y^{(k)}} \|_2^2  \\
     & \quad - \theta_k q \<\nu_{e_k, i_k}^{\y^{(k)}}, \v^{(k)}_{e_k, i_k} - \y^{(k)}_{e_k, i_k} \> 
 \end{align*}
\end{restatable}
The proof is similar to what is given by \citet{lee2013efficient}, but since we consider the non-strongly convex case and slightly different definitions, we give a full proof in Appendix~\ref{sec::technical-proofs} for completeness. We can use this fact to show a rate of convergence on the dual objective.
\begin{lemma}\label{lem::acc-dual-conv}
    For the random estimate sequence in (\ref{eq::estimate-sequence}), let $\{ \lambdad^{(k)}\}_{k = 0}^K$ and $\{\y^{(k)}\}_{k = 0}^K$ be defined as in Algorithm~\ref{alg::accel-emp} with $\lambdad^{(0)} = 0$. Then, the dual objective error in expectation can be bounded as 
    $
    \E [ L(\lambdad^{(k)}) - L(\lambdad^*)] \leq \frac{G(\eta)^2}{(k + 2)^2}$,
   	where $G(\eta) := 24 md (m + n) ( \sqrt \eta \|C\|_\infty  + \frac{\log d}{\sqrt \eta} )$. 
\end{lemma}
\begin{proof}
It suffices to show that the sequence in (\ref{eq::estimate-sequence}) with the definitions of $\{ \lambdad^{(k)}\}_{k = 0}^K$ and $\{\y^{(k)}\}_{k = 0}^K$ satisfies $\E[ L(\lambdad^{(k)})]  \leq  \min_{\lambdad} \E \left[ \phi_k(\lambdad) \right]$. To do this, we will use induction and Lemma~\ref{lem::emp-improvement}. Note that $\min_{\lambdad} \phi_k(\lambdad) = \omega_k$. The base case holds trivially with $\E[ \omega_0] = L(\lambdad^{(0)})$. Suppose $\E \left[ L(\lambdad^{(k)}) \right] \leq \E \left[ \omega_{k} \right]$ at iteration $k$. For $k + 1$, we have
\begin{align*}
    & \E [ \omega_{k + 1} ]  \\
    & \geq  (1 - \theta_k) \E [L(\lambdad^{(k)})] + \theta_k\E [ L(\y^{(k)})] \\ 
    &\quad - \E \left[  \frac{(\theta_kq)^2}{2\gamma_{k+1}} \| \nu_{i_k, e_k}^{\y^{(k)}} \|_2^2   - \theta_k q \<\nu_{e_k, i_k}^{\y^{(k)}}, \v^{(k)} - \y^{(k)} \> \right].
\end{align*}
The above inequality uses the inductive hypothesis. Then,
\begin{align*}
    & \geq \E \left[ L(\y^{(k)}) - \frac{(\theta_k q)^2}{2\gamma_{k+1}} \| \nu_{i_k, e_k}^{\y^{(k)}} \|_2^2\right] \\
    & \quad + (1 - \theta_k) \E \left[ \< \nabla L(\y^{(k)}), \lambdad^{(k)} - \y^{(k)}   \> \right] \\
    & \quad +  \theta_k \E \left[ \< \nabla L(\y^{(k)}), \v^{(k)} - \y^{(k)}   \> \right] \\
    & = \E \left[ L(\y^{(k)}) - \frac{\theta_k^2 q}{2\gamma_{k+1}} \| \nabla L(\y^{(k)}) \|_2^2\right].
\end{align*}
This second inequality uses convexity of $L$ and then applies the law of total expectation to the dot products, conditioning on randomness $\{e_s, i_s \}_{s = 1}^{k - 1}$. The last identity uses the definition of $\y^{(k)}$ and  total expectation again on the norm.

Using Lemma~\ref{lem::emp-improvement} and the definition of $\lambdad^{(k + 1)}$ from Algorithm~\ref{alg::accel-emp}, we have
\begin{align*}
    E [ L(\lambdad^{(k + 1)} ]     & \leq \E [L(\y^{(k)}) - \frac{1}{4\eta}\| \nu_{e_k, i_k}^{\y^{(k)}} \|_2^2 ] \\
    & = \E [L(\y^{(k)}) - \frac{1}{4q \eta }\| \nabla L(\y^{(k)}) \|_2^2 ]
\end{align*}
The equality uses the law of total expectation. Therefore, taking $\theta_k^2 := \frac{\gamma_{k + 1}} {2 q^2 \eta}$ ensures that
$\min_{\lambdad} \E [ \phi_{k + 1} (\lambdad)] \geq \E [ \omega_{k  +1} ] \geq E [ L(\lambdad^{(k + 1)} ]$. 
By setting $\gamma_0 = 2 q^2 \eta$, we ensure that $\theta_k$ need only satisfy $\theta_k^2 = (1 - \theta_k) \theta^2_{k - 1}$, which occurs in Algorithm~\ref{alg::accel-emp}.
    From \citet[Lemma 2.2.4]{nesterov2018lectures}, this choice of $\gamma_0$ and $\theta_k$ ensures
        $\delta_k \leq \left( 1 + \frac{k}{q}\sqrt{\frac{\gamma_0}{8 \eta}} \right)^{-2} = \frac{4}{(k + 2)^2}$.
    From (\ref{eq::estimate-sequence-bound}) and the definition of $\phi_0$, we get
    \begin{align*}
        \E [ L(\lambdad^{(k)}) - L(\lambdad^*)] \leq \frac{4 L(0) - 4 L(\lambdad^*)  + 16m^2 \eta \| \lambdad^*\|^2_2 }{(k + 2)^2}.
    \end{align*}
    The proof that the numerator can be bounded by $G(\eta)^2$ is deferred to the appendix, Lemma~\ref{lem::g-bound}.
\end{proof}


\vspace{-.2cm}
\subsubsection{Approximation Error due to Entropy}

Recall that our end goal is to ensure that $\widehat \mu$, the projection of $\mu^{\widehat \lambdad}$ onto $\LL_2$, is expected $\epsilon$-optimal for $\epsilon > 0$. To show this, we need to develop some relations which we outline here for brevity but state formally and prove in Appendix~\ref{sec::approx-bound-lemma-proofs}.
The first relation is how close $\widehat \mu \in \LL_2$ is to the algorithm's output $\mu^{\widehat\lambdad}$, which is in the slack polytope $\LL_2^{\nu^{\widehat\lambdad}}$. This is essentially a direct extension of \citet[Lemma 7]{altschuler2017near}, which tells us that we can bound the projection by the norm of the slacks. Then, we show that there exists a point $\widehat \mu^{*}$  in the slack polytope $\LL_2^{\nu^{\widehat\lambdad}}$ that is close to the optimal point $\mu^* \in \LL_2$ with respect to the norm of the slacks also.
We use the relations to conclude a bound on the original relaxed problem (\ref{eq::original}) for any realization of $\widehat \mu$. 
\begin{restatable}{proposition}{propapproxbound}
\label{prop::approx-bound}
Let $\mu^* \in \LL_2$ be optimal, $\lambdad \in \R^{r_D}$, $\widehat \mu = \emph{\proj}(\mu^{\lambdad}, 0) \in \LL_2$, and $\delta = \max_{e \in E, i \in e} \|\nu_{e, i}^{ \lambdad}\|_1$. The following inequality holds:
\begin{align*}
    & \< C, \widehat \mu - \mu^*\>  \leq 16 (m + n) d \|C\|_\infty \delta  \\  & \quad + 4 \|C\|_\infty\sum_{e \in E, i \in e}  \| \nu_{e,i}^{ \lambdad}\|_1 
     + \frac{n \log d + 2m \log d}{\eta}.
\end{align*}
\end{restatable}

\vspace{-.5cm}
\subsubsection{Completing the Proof}


\begin{proof}[Proof of Theorem~\ref{th::accel} for \emph{\emp{}}]
Let $\widehat \lambdad$ be the output from Algorithm~\ref{alg::accel-emp} after $K$ iterations.
From Lemma~\ref{lem::emp-improvement}, we can lower bound the result in Lemma~\ref{lem::acc-dual-conv} with $\frac{1}{4\eta }\E [ \| \nu_{e, i}^{\widehat \lambdad}\|_1^2 ]  
\leq \frac{G(\eta) }{(K + 2)^2}$
for all $e \in E, i \in e$. Then, for $\epsilon' > 0$, we can ensure that
\begin{align*}
    \E [ \| \nu_{e, i}^{\widehat \lambdad}\|_1 ] & \leq \epsilon'  & & \text{and} & &  
    \E \sum_{e \in E, i \in e}  \| \nu_{e, i}^{\widehat \lambdad}\|_1^2 & \leq 2m (\epsilon')^2
\end{align*}
in $K = \frac{\sqrt{4 \eta }G(\eta)}{\epsilon'}$ iterations.
Let $\widehat\mu \in \LL_2$ be the projected version of $\mu^{\widehat \lambdad}$.
Taking the expectation of both sides of the result in Proposition~\ref{prop::approx-bound} gives us
\begin{align*}
    \E [\< C, \widehat \mu-\mu^*\>] & \leq \| C\|_\infty  \left( 16( m + n) d \E [\delta]  + 8m \epsilon' \right) \\
    & \quad + \frac{n \log d + 2 m \log d}{\eta},
\end{align*}
where $\E \left[ \delta \right]^2  \leq  \E [ \delta^2 ]  \leq \E  \sum_{e \in E, i \in e} \| \nu_{e, i}^{\widehat \lambdad}\|_1^2   \leq 2m (\epsilon')^2$.
Then we can conclude
\begin{align*}
        \E [ \< C, \widehat \mu-\mu^*\>] & \leq   16\sqrt{2m} ( m + n) d \| C\|_\infty \epsilon'  \\
        & \quad + 8m \| C\|_\infty \epsilon'  + \frac{n \log d + 2 m \log d}{\eta} \\
        & \leq 24\sqrt{2m} ( m + n) d \| C\|_\infty \epsilon' \\
        & \quad + \frac{n \log d + 2 m \log d}{\eta}.
    \end{align*}
    Therefore, $\widehat \mu$ is expected $\epsilon$-optimal with $\eta$ as defined in the statement and $\epsilon' = \frac{\epsilon}{48\sqrt{2m} ( m + n) d \| C\|_\infty}$. Substituting these values into $K$ and $G(\eta)$ yields the result.
\end{proof}

\vspace{-.5cm}
\section{Rounding to Integral Solutions} \label{sec::rounding}

\begin{figure*}
	\centering
	\includegraphics[width=.32\textwidth]{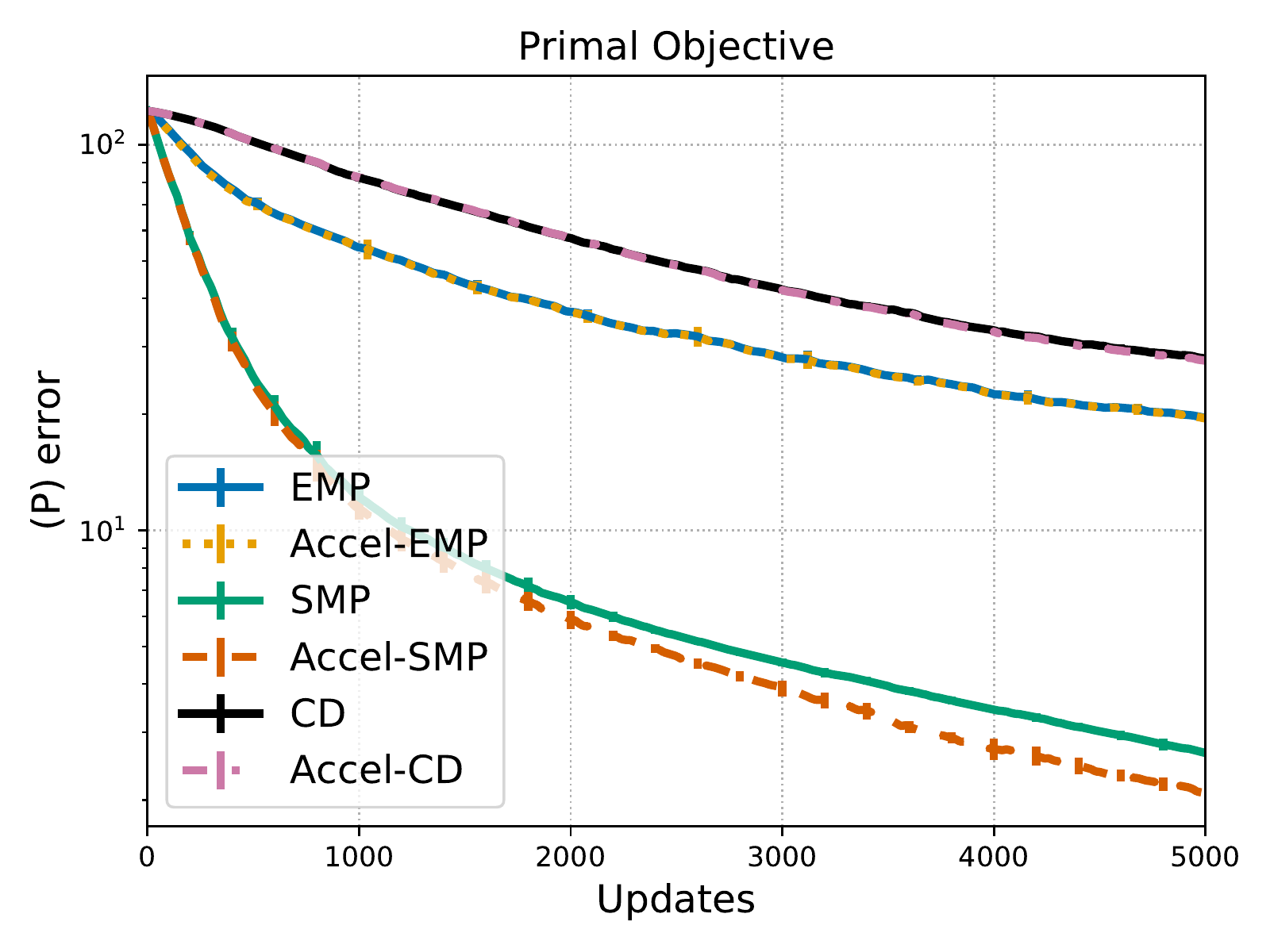}
	\includegraphics[width=.32\textwidth]{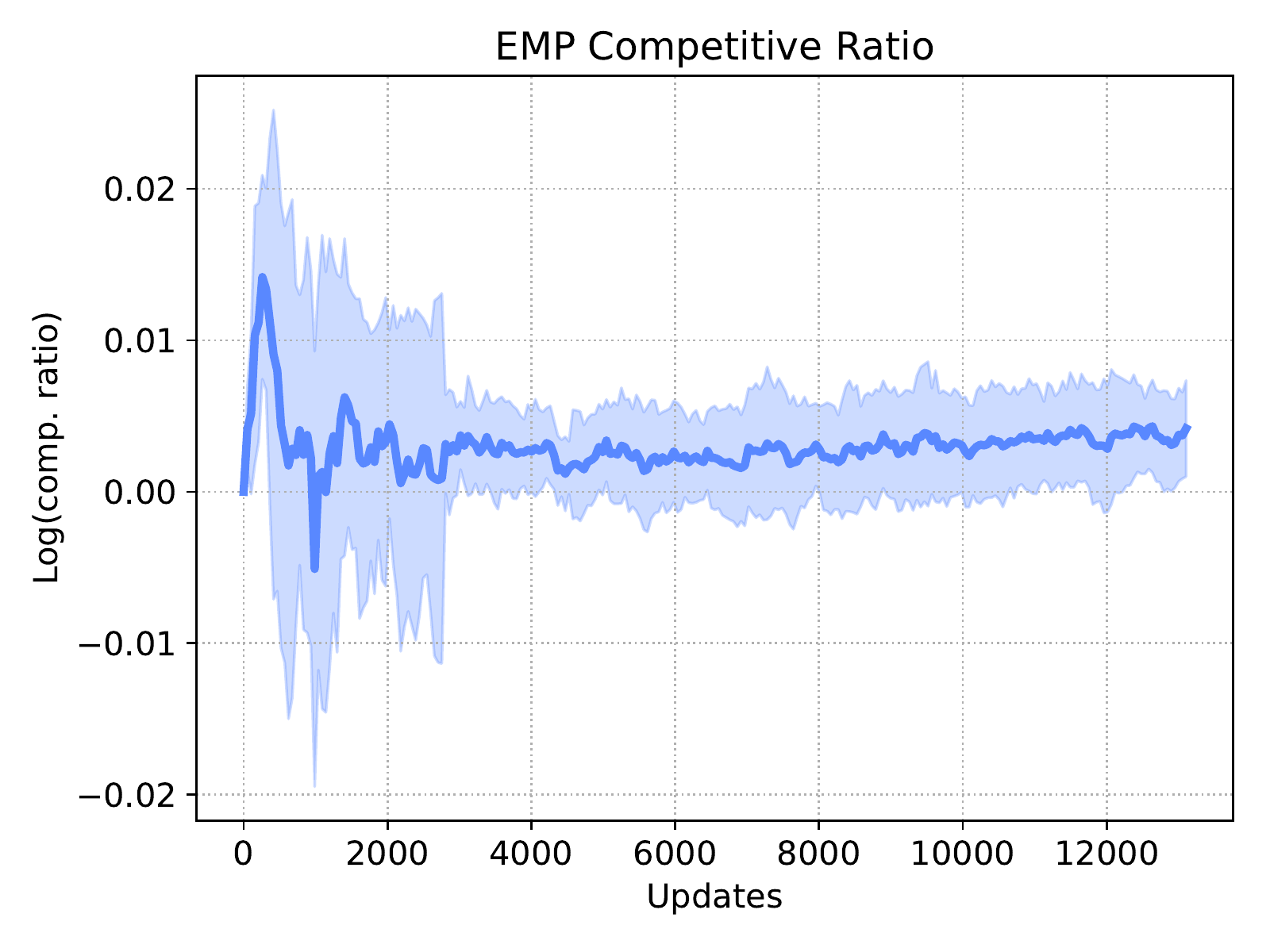}
	\includegraphics[width=.32\textwidth]{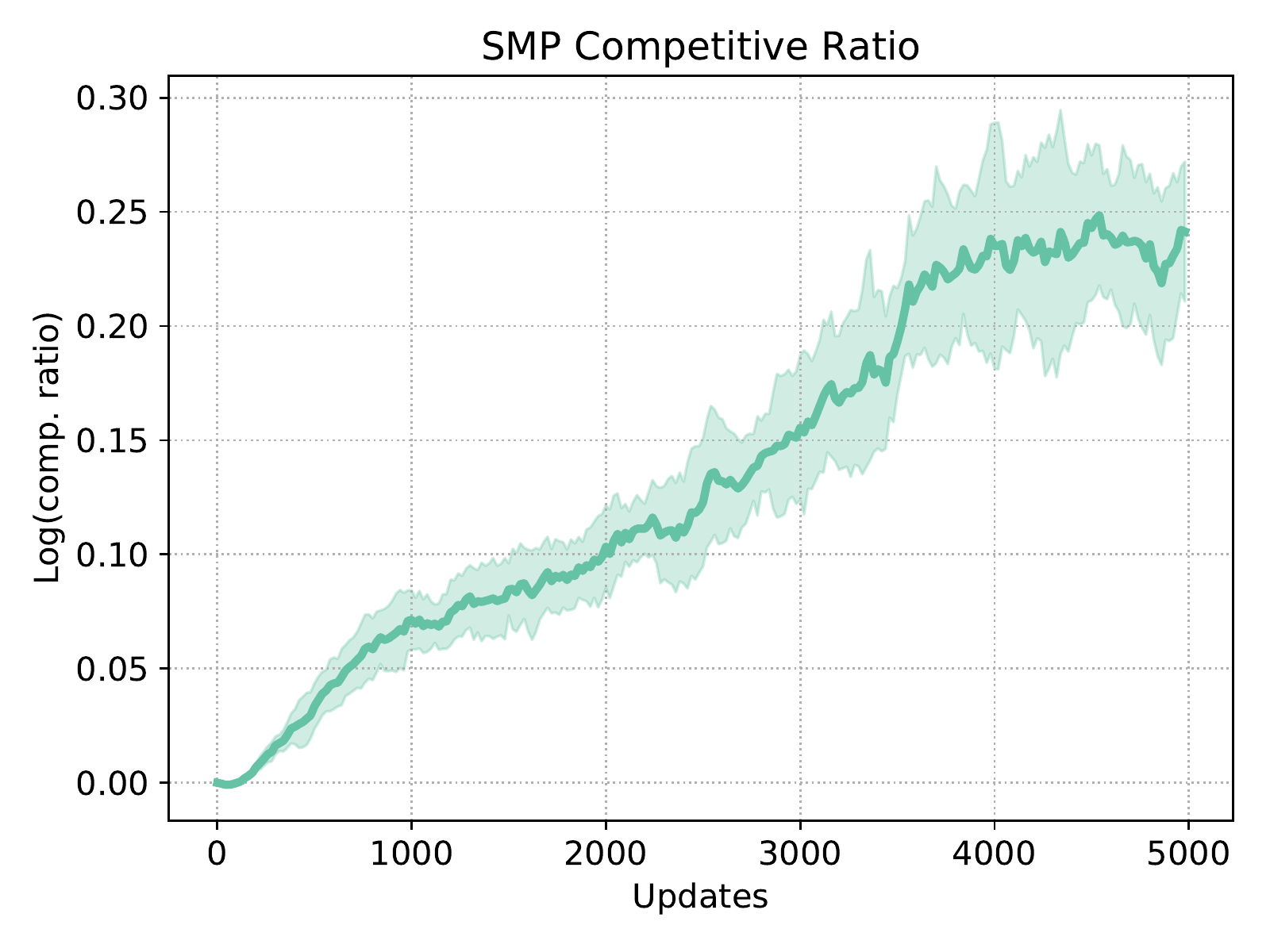}

	\caption{
	\textit{(Left column)}: The original primal objective (\ref{eq::original}) on a log scale is compared for the standard algorithms and their accelerated variants, as well as standard (accelerated) coordinate descent, over 10 trials on an Erd\H{o}s-R\'enyi random graph with $n = 100$. Error bars denote standard deviation. 
	\textit{(Center and right columns)}: The log-competitive ratio of \emp{} and \smp{} with respect to their accelerated variants, \aemp{} and \asmp{}, on the primal objective (\ref{eq::original}). 
\vspace{-.3cm}
}
	\label{fig:plot_curve}
\end{figure*}

Inspired by recent results regarding the approximation error achieved by entropy regularization in linear programming \cite{weed2018explicit}, we are able to derive rounding guarantees for our algorithms under the assumption that the LP relaxation is tight and the solution is unique. We use a simple rounding scheme: for any $\mu$ that may not lie in $\LL_2$, 
$(\mathrm{round}(\mu))_i = \arg \max_x \mu_i(x)$. The main challenge in achieving the results we present here is surpassing the difficulty in obtaining bounds for the $l_1$ distance between $\widehat{\mu} = \mu^{\widehat{\lambdad}}$, the candidate solution obtained from the final iterate $\widehat{\lambdad}$ resulting from our algorithms, and $\mu^*$, the optimal solution of (\ref{eq::original}). Define $\mu^{\lambdad^*}$ be the solution to the regularized problem where $\lambdad^* \in \Lambda^*$.

We proceed in two steps. First, we bound the approximation error $\| \mu^{\lambdad^*} - \mu^*\|_1$ using recent results on the quality of solutions for entropy regularized $\LL_2$ \citep{lee2019approximate}. Then we bound the optimization error of $\widehat \mu$ using the results derived in the previous section.
The proof and a comparison to standard \emp{} are in Appendix~\ref{section::rounding_guarantees}.
 Let $\deg$ denote the maximum degree of the graph $G$ and define $\Delta$ as the suboptimality gap of the LP over $\LL_2$.

\begin{restatable}{theorem}{throunding}

Let $\delta \in (0,1)$. If $\LL_2$ is tight for potential vector $C$ and there exist a unique solution to the MAP problem and $\eta = \frac{ 16(m+n)(\log(m+n)+ \log(d))}{\Delta}$
then with probability $1-\delta$ in at most
\begin{equation*}
O \left( \frac{d^3 m^7 \deg^2 \| C\|_\infty^2 \log^2 dm}{\delta \Delta} \right)
\end{equation*}
iterations of \emph{\aemp{}}, $\mathrm{round}(\widehat{\mu})$ is the optimal MAP assignment.
\end{restatable}
\vspace{-.3cm}

\section{Numerical Experiments}

 Our goal now is to understand the empirical differences between the above algorithms and also where certain theoretical guarantees can likely be improved, if at all. To this end, we compare the convergence rates of \emp{} and \smp{} and their accelerated variants on several synthesized Erd\H{o}s-R\'enyi random graphs.  First, we constructed a graph with $n = 100$ vertices and then generated edges between each pair of vertices with probability $1.1 \frac{\log n }{n}$. We considered the standard multi-label Potts model with $d = 3$ labels. The cost vector $C$ was initialized randomly in the following way: $C_i(x_i)  \sim \unif([-0.01, 0.01])$, $\forall x_i \in \chi$ and $C_{ij}(x_i, x_j) \sim \unif(\{-1.0, 1.0\})$,  $\forall x_i, x_j \in \chi$.

 We consider two different metrics.
(1) The first is the original primal objective value (\ref{eq::original}). This metric computes the objective value of the projection $\widehat \mu = \proj(\mu^{\lambdad}, 0)$.
(2) The second reports the log-competitive ratio between the standard and accelerated variants. The competitive ratio is computed as $\log \left( \frac{\< C, \widehat\mu_{\emp} - \mu^* \>  }{ \<C, \widehat \mu_{\aemp} - \mu^* \>}  \right)$, where $\widehat \mu_{\emp}$ and $\widehat \mu_{\aemp}$ are the projections due to \proj{}. Positive values indicate that the accelerated variant has lower error.
We implemented the four message passing algorithms exactly as they are described in  Algorithms~\ref{alg::semp}, \ref{alg::accel-emp}, and \ref{alg::accel-smp}. We also implemented block-coordinate descent and its accelerated variant as baselines \citep{lee2013efficient} with a step size of $1/\eta$. 
Each algorithm used $\eta = 1000$ over 10 trials, measuring means and standard deviations. We computed the ground-truth optimal value of (\ref{eq::original}) using a standard solver in CVXPY.

Figure~\ref{fig:plot_curve} depicts convergence on the primal objective in the left column. \smp{} achieves convergence in the fewest iterations, and \emp{} converges faster than coordinate descent. Interestingly, we find that the accelerated variants, including accelerated coordinate descent, appear to have marginal improvement on this metric. However, the competitive ratio figures confirm that the accelerated variants are consistently faster, especially for \smp{}. 
 These results suggest that, at least for this particular problem, the upper bounds for standard algorithms may be overly conservative.
It would be interesting to investigate tighter bounds for the standard algorithms in future work. Further details can be found in the appendix.
\vspace{-.3cm}

\section{Conclusion}
We analyze the convergence of message passing algorithms on the MAP inference problem over $\LL_2$. In addition to providing a novel rate of convergence rate for standard schemes derived from entropy regularization, we show that they can be directly accelerated in the sense of Nesterov with significant theoretical improvement.
In future work it would be interesting to consider accelerating greedy message passing algorithms; however, \citet{lu2018accelerating} suggest that, despite empirical success, proving accelerated rates for greedy methods is an open question even in the basic coordinate descent case.
The tightness of the presented guarantees is also an open question, motivated by the empirical results here. Finally, we conjecture that reductions from the active area of optimal transport could yield novel, faster algorithms.



\bibliography{main.bib}
\bibliographystyle{icml2020}

\newpage
\onecolumn

\appendix

\onecolumn

\section{Numerical Experiments Details}

All experiments were run on an Intel Core i7 processor with 16 GB RAM.
 We consider two different metrics.

\begin{enumerate}
\item The first metric is the original primal objective value (\ref{eq::original}), which is the actual value we wish to minimize. Since the optimization is done in the dual variables, we use \proj{} to project the point $\mu^{\lambdad}$ onto $\LL_2$ and report the projection's function value, additively normalized with respect to the optimal value. 
\item The second metric emphasizes the first by reporting the log-competitive ratio between the standard and accelerated variants of the the algorithms. The competitive ratio is computed as $\log \left( \frac{\< C, \widehat\mu_{\emp} - \mu^* \>  }{ \<C, \widehat \mu_{\aemp} - \mu^* \>}  \right)$, where $\widehat \mu_{\emp}$ and $\widehat \mu_{\aemp}$ are the projections due to \proj{} of the outputs of \emp{} and \aemp{}, respectively, and $\mu^*$ is a minimizer over $\LL_2$. The same is computed for \smp{} and \asmp{}. Thus, positive values at a given time indicate that the accelerated variant has lower error on the original objective.
\end{enumerate}

\begin{wrapfigure}{R}{0.4\textwidth}
\centering
	\includegraphics[width=.4\textwidth]{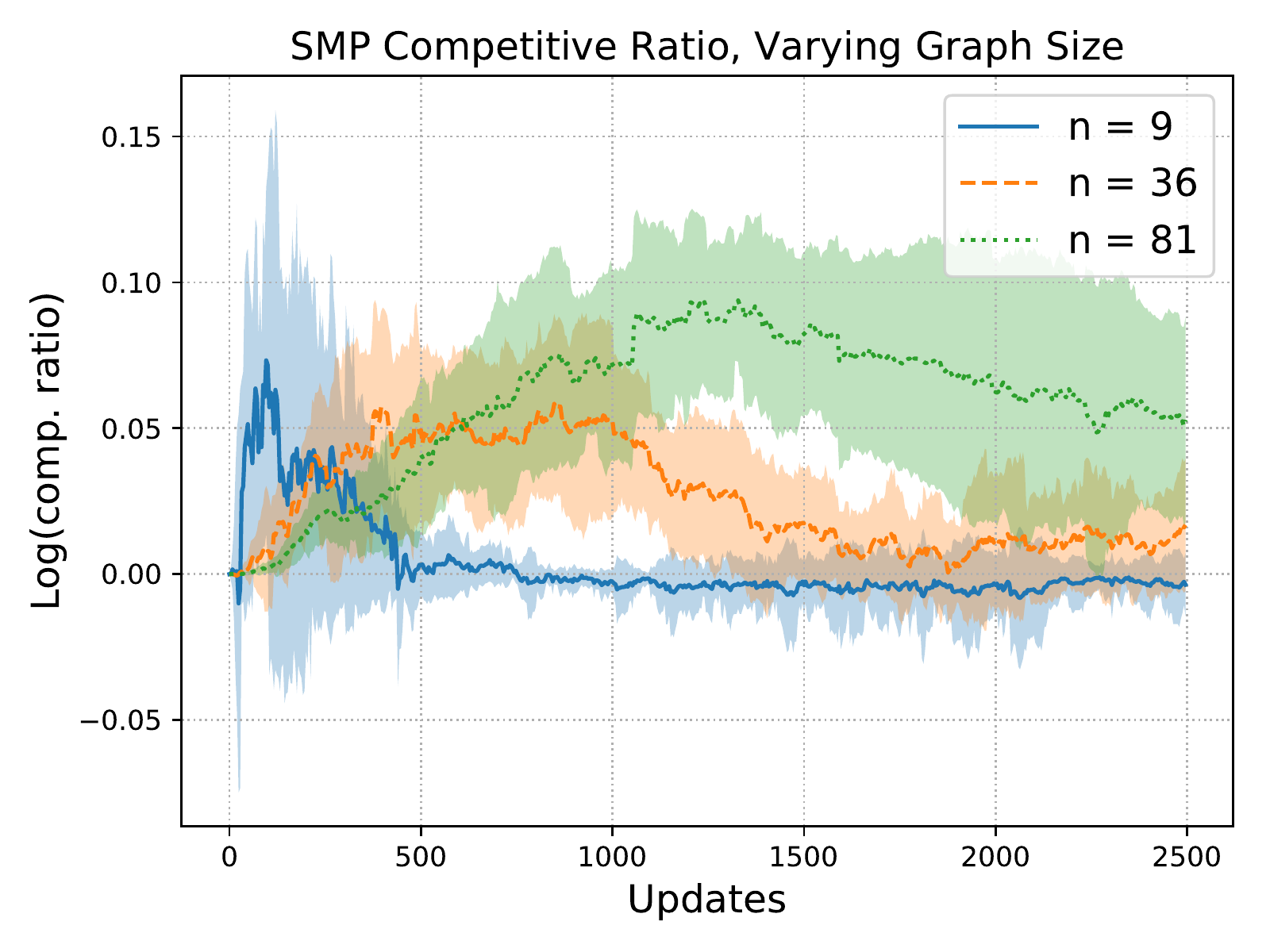}
\caption{
	The competitive ratio of \smp{} with respect to \asmp{} on (\ref{eq::original}) is compared across random graphs of varying sizes $n = 9$, $36$, and $81$.
	}
	\label{fig:sizes}
\end{wrapfigure}

\paragraph{Message Passing Algorithms}

We implemented the message passing algorithms with their update rules exactly as prescribed in Algorithms~\ref{alg::semp}, \ref{alg::accel-emp}, and \ref{alg::accel-smp}. The algorithms are compared with respect to the number of updates (i.e. iterations); however, we note that the cost of each update is greater for \smp{} and \asmp{} since they both require computing slacks of the entire neighborhood surrounding a give vertex. 

\paragraph{Block-Coordinate Methods}
In addition to studying the empirical properties of the message passing algorithms, we present a supplementary empirical comparison with block-coordinate descent and its accelerated variant \citep{lee2013efficient}. The purpose of this inclusion is to standardize how much we expect acceleration to improve the algorithms. We chose a stepsize of $1 / \eta$. We note that each update in block-coordinate descent is essentially as expensive as an update of \emp{}.

%

This choice of cost vectors $C$ ensures that vertices cannot be trivially set to their minimal vertex potentials to achieve a reasonable objective value; the MAP algorithm must actually consider pairwise interactions. We evaluated each of the four algorithms on the same graph with $\eta = 1000$. Due to the inherent stochasticity of the randomized algorithms, we ran each one 10 times and took the averages and standard deviations. Since the graphs are small enough, we computed the ground-truth optimal value of (\ref{eq::original}) using a standard LP solver in CVXPY.

In order to understand the effect of the graph size on the competitive ratio between the standard and accelerated algorithms, we generated random graphs of sizes $n = 9$, $36$, and $81$ with the same randomly drawn edges and cost vectors. We ran \smp{} and \asmp{} for a fixed number of iterations over 10 random trials and computed the average log competitive ratio. Again, we used $\eta = 1000$. Figure~\ref{fig:sizes} shows that \asmp{} runs faster in all regimes, especially at beginning, and then the performance improvement eventually tapers after many iterations.




\section{Omitted Proofs and Derivations of Section~\ref{sec::setup}}

\subsection{Proof of Proposition~\ref{prop::dual-obj}}

Recall that the primal objective is to solve the following:
	\begin{align}\label{eq:obj}\tag{obj}
	\minimize \quad  \< C, \gam\> - \frac{1}{\eta} H(\gam) \quad  \st \quad \gam \in \LL_2,
	\end{align}
	where
	\begin{align*}
	H(\gam) = - \sum_{i \in V} \sum_{x_i \in \chi}  \gam_i(x_i)( \log \gam_i(x_i) - 1) - \sum_{e \in E} \sum_{x_e \in \chi^2} \gam_{e} (x_e) (\log \gam_e(x_e) - 1).
	\end{align*}
	Though it is not strictly necessary, we will also include a normalization constraint on the pseudo-marginal edges, which amounts to $\sum_{x_e} \gam_e(x_e) = 1$ for all $e \in E$. The Lagrangian is therefore
	\begin{align*}
	\L(\gam, \lambdad, \xi) = \<  C, \gam\> -\frac{1}{\eta} H(\gam)  + \sum_{e \in E, i \in e} \lambdad_{e,i}^\top (S_{e, i} - \gam_i) + \sum_{i \in V} \xi_i (\sum_{x_i} \gam_i(x_i) - 1) + \sum_{e\in E} \xi_e ( \sum_{x_e} \gam_e(x_e) - 1)
	\end{align*}
	Taking the derivative w.r.t~$\gam$ yields
	\begin{align}
	\frac{\partial \L(\gam,\lambdad, \xi) }{\partial \gam_i(x_i)} & =  C_{i}(x_i) + \frac{1}{\eta} \log \gam_i(x_i) + \xi_i - \sum_{e \in N_i} \lambdad_{e, i}(x_i) \\
	\frac{\partial \L(\gam,\lambdad, \xi) }{\partial \gam_e(x_e)} & =   C_{e}(x_e) + \frac{1}{\eta}\log \gam_e(x_e) + \xi_{e} + \sum_{i \in e} \lambdad_{e, i}((x_e)_i).
	\end{align}
	Here we are using $(x_e)_i$ to denote selecting the label associated with endpoint $i \in V$ from the pair of labels $x_e \in \chi^2$. The necessary conditions for optimality imply that
	\begin{align}
	\mu_{i}^{\lambdad, \xi}(x_i) & = \exp\left(-\eta C_{i} (x_i)  - \eta \xi_i + \eta \sum_{e \in N_i} \lambdad_{e, i}(x_i) \right) \\
	\mu_{e}^{\lambdad, \xi}(x_e) & = \exp\left(-\eta C_{e} (x_e)  - \eta \xi_e - \eta \sum_{i \in e} \lambdad_{e, i}((x_e)_i) \right),
	\end{align}
	where we use the superscripts to show that these optimal values are dependent on the dual variables, $\lambdad$ and $\xi$. The dual problem then becomes
	\begin{align*}
	\maximize_{\lambdad, \xi} \quad -\frac{1}{\eta}\sum_{i}\sum_{x_i} \mu_i^{\lambda, \xi}(x_i) - \frac{1}{\eta}\sum_{c}\sum_{x_c} \mu_c^{\lambdad, \xi}(x_c)  - \sum_{i \in V}  \xi_i - \sum_{e \in E}  \xi_e
	\end{align*}
	Note that we can solve exactly for $\xi$ as well, which simply normalizes the individual pseudo-marginals for each edge and vertex so that
	\begin{align*}
	    \xi_i & = \frac{1}{\eta} \log \sum_{x_i} \exp\left(-\eta C_{i} (x_i)  - \eta \xi_i + \eta \sum_{e \in N_i} \lambdad_{e, i}(x_i)\right) \\
	    \xi_e & = \frac{1}{\eta} \log \sum_{x_e}\exp\left(-\eta C_{e} (x_e)  - \eta \xi_e - \eta \sum_{i \in e} \lambda_{e, i}((x_e)_i)\right) 
	\end{align*}
Plugging this into $\mu^{\lambdad, \xi}$ ensures that each local vertex and edge distribution is normalized to $1$. Therefore the final objective becomes
\begin{align*}
    \minimize_{\lambdad} \quad & \frac{m + n}{\eta} + \frac{1}{\eta} \sum_{i \in V} \log \sum_{x_i} \exp\left(-\eta C_{i} (x_i)  - \eta \xi_i + \eta \sum_{e \in N_i} \lambdad_{e, i}(x_i)\right) \\
    & \quad + \frac{1}{\eta} \sum_{e \in E} \log \sum_{x_e}\exp\left(-\eta C_{e} (x_e)  - \eta \xi_e - \eta \sum_{i \in e} \lambda_{e, i}((x_e)_i)\right),
\end{align*}
and we can ignore the constant.

\subsection{Entropy-Regularized Message Passing Derivations}

In this section, we derive the standard message passing algorithms that will be the main focus of the paper. Both come from simply computing the gradient and choosing additive updates to satisfy the optimality conditions directly.

\propempupdate*

\begin{proof}
    	From (\ref{eq::derivative}), the partial gradient of $L$ with respect to coordinate $(e, i, x_i)$ yields the following necessary and sufficient optimality condition:
\begin{align*}
S_{e, i}^{\lambdad}(x_i) = \mu^{\lambdad}_i(x_i).
\end{align*}
Suppose that $\lambdad'$ satisfies this condition, and thus minimizes $L_{e, i}(\cdot; \lambdad)$. We can decompose $\lambdad'$ at coordinate $(e, i, x_i)$ additively as $\lambdad'_{e, i}(x_i) = \lambdad_{e, i}(x_i) + \delta_{e, i}(x_i)$. From the definition of $\mu^{\lambdad}$, the optimality condition becomes
\begin{align*}
\exp( 2 \eta \delta_{c, i}(x_i)  ) = \frac{\sum_{x_j \in \chi} \mu^{\lambdad}_e(x_i, x_j)}{\mu^{\lambdad}_i(x_i)}
\end{align*}
Rearranging to find $\delta_{e, i}(x_i)$ and then substituting into $\lambdad'_{e, i}(x_i)$ yields the desired result.
\end{proof}

Now, we can derive a lower bound on the improvement on the dual objective $L$ from applying an update of \emp{}.
\lemempimprovement*

\begin{proof}
	Let $\tilde \mu$ denote the unnormalized marginals. From the definition of $L$,
	\begin{align*}
	L(\lambdad) - L(\lambdad') & =  \frac{1}{\eta}\log  \sum_{x_i} \exp\left( -\eta C_{i}(x_i) + \eta \sum_{e \in N_i} \lambdad_{e, i}(x_i)\right) + \frac{1}{\eta}\log \sum_{x_e} \exp\left( -\eta C_{e}(x_e) - \eta \sum_{i \in e} \lambdad_{e, i}(x_i)\right)  \\
	& \quad  - \frac{1}{\eta}\log  \sum_{x_i} \exp\left( -\eta C_{i}(x_i) + \eta \delta_{e, i}(x_i) + \eta \sum_{e \in N_i} \lambdad_{e, i}(x_i)\right) \\
	& \quad  - \frac{1}{\eta}\log \sum_{x_e} \exp\left( -\eta C_{e}(x_e) - \eta \delta_{e, i}(x_i) -\eta  \sum_{i \in e} \lambdad_{e, i}(x_i)\right)
	\end{align*}
	Define $\tilde \mu_{i}^{\lambdad}(x_i) = \exp\left( -\eta C_{i}(x_i) + \sum_{e \in E} \lambdad_{e, i}(x_i)\right)$ and $\tilde \mu_{e}(x_e) = \exp\left( -\eta C_{e}(x_e) - \sum_{i \in e} \lambdad_{e, i}(x_i)\right)$. The cost difference can then be written as
	\begin{align*}
	L(\lambdad) - L(\lambdad') & = - \frac{1}{\eta} \log \left(\sum_{x_i} \frac{\tilde \mu_{i}^{\lambdad}(x_i)e^{\eta\delta_{e, i}(x_i)}}{\sum_{x_i'}\tilde \mu_{i}^{\lambdad}(x_i')} \right) - \frac{1}{\eta} \log \left( \sum_{x_e} \frac{\tilde \mu_{e}^{\lambdad}(x_e)e^{-\eta\delta_{e, i}(x_c)}}{\sum_{x_e'} \tilde \mu_{e}^{\lambdad}(x_e')} \right) \\
	& = - \frac{1}{\eta}\log \left(\sum_{x_i} \mu_i^{\lambdad}(x_i) \sqrt{\frac{S_{e, i}(x_i)}{\mu_{i}^{\lambdad}(x_i)}}\right) - \frac{1}{\eta}\log \left( \sum_{x_e} \mu_{e}^{\lambdad}(x_e)\sqrt {\frac{\mu_{i}^{\lambdad}(x_i)  }{ S_{e, i}^{\lambdad} (x_i)}}  \right) \\
	& = -\frac{2}{\eta} \log \left( \sum_{x_i} \sqrt { S_{e, i}^{\lambdad}(x_i) \mu_i^{\lambdad}(x_i) } \right)
	\end{align*}
	Note that the right-hand contains the Bhattacharyya coefficient $BC(p, q) := \sum_{i} \sqrt{p_i q_i}$ which has the following relationship with the Hellinger distance: $BC(p, q) = 1 - h^2(p, q)$.
	
	The inequality then follows from exponential inequalities:
	\begin{align*}
	L(\lambdad) - L(\lambdad')  = - \frac{2}{\eta}\log (1 - h^2(S_{e, i}^{\lambdad}, \mu_i^{\lambdad})) \geq -\frac{2}{\eta} \log \exp( {-h^2(S_{e, i}^{\lambdad}, \mu_i^{\lambdad})) } = \frac{2}{\eta} h^2(S_{e, i}^{\lambdad}, \mu_i^{\lambdad})
	\end{align*}
	Furthermore, the Hellinger inequality gives us
	\begin{align*}
	\frac{1}{4}\| p - q\|_1^2 \leq 2 h^2(p, q).
	\end{align*}
	We conclude the result by applying this inequality with $p = S_{e, i}^{\lambdad}$ and $q = \mu_i^{\lambdad}$.
\end{proof}

\propsmpupdate*

\begin{proof}
The optimality conditions require, for all $e \in N_i$, 
	\begin{align*}
	S_{e, i}^{\lambdad'}(x_i) = \gam_{i}^{\lambdad'}(x_i),
	\end{align*}
	which implies that
	\begin{align*}
	S_{e, i}^{\lambdad}(x_i) = \gam_{i}^{\lambdad}(x_i)\exp\left(\eta \delta_{e, i}(x_i) + \eta\sum_{e' \in N_i} \delta_{e' , i}(x_i) \right),
	\end{align*}
	where $\delta_{e, i}(x_i) = \lambdad'_{e, i}(x_i) -\lambdad_{e, i}(x_i)$. Then, let $e_1, e_2 \in N_i$. At optimality, it holds that
	\begin{align*}
	\frac{S_{e_1, i}^{\lambdad}(x_i)}{S_{e_2, i}^{\lambdad}(x_i)} = \frac{\exp( {\eta\delta_{e_1, i}(x_i)})} {\exp( {\eta\delta_{e_2, i}(x_i)}) }
	\end{align*}
	Substituting each $\delta_{e_2, i}(x_i)$ in terms of $\delta_{e_1, i}(x_i)$, we then have
	\begin{align*}
	S_{e, i}^{\lambdad}(x_i) =\gam_{i}^{\lambdad}(x_i) \exp({\eta\delta_{e, i}(x_i)}) \prod_{e'\in N_i}  \frac{S_{e', i}^{\lambdad}(x_i)}{S_{e, i}^{\lambdad}(x_i)}\exp({\eta\delta_{e, i} (x_i) }).
	\end{align*} 
	Collecting and then rearranging the above results in
	\begin{align*}
    \exp({(|N_i|+1)\eta\delta_{e, i}(x_i)}) & = (S_{e, i}^{\lambdad}(x_i))^{|N_i|+1}  \left(\gam_{i}^{\lambdad}(x_i) \prod_{e' \in N_i} {S_{e', i}^{\lambdad}(x_i)}\right)^{-1}.
	\end{align*}
	In additive form, the update equation is
	\begin{align*}
	\delta_{e, i}(x_i) = \frac{1}{\eta}\log S_{e, i}^{\lambdad} (x_i) - \frac{1}{\eta(|N_i| + 1)} \log \left( \gam_{i}^{\lambdad}(x_i)\prod_{e' \in N_i} {S_{e', i}^{\lambdad}(x_i)} \right).
	\end{align*}
\end{proof}

\section{Omitted Proofs of Technical Lemmas of Section~\ref{sec::main-results}}\label{sec::technical-proofs}

\subsection{Proof of Random Estimate Sequences Lemma~\ref{lem::est-seq}}

\lemestseq*

\begin{proof}
	First we show that it is an estimate sequence by induction. Clearly this holds for the base case $\phi_0$ when $\delta_0 = 1$. Then, we assume the inductive hypothesis that $\E[\phi_k(\lambdad)] \leq (1 - \delta_k) L(\lambdad)+ \delta_k \phi_0(\lambdad)$. From, here we can show
	\begin{align*}
	\E [ \phi_{k + 1}(\lambdad)] & = (1 - \theta_k) \E [\phi_k(\lambdad)] + \theta_k \E \left[ L(\y^{(k)}) -  \< q \nu_{e_k, i_k}^{\y^{(k)}}, \lambdad_{e_k, i_k} -  \y^{(k)}_{e_k, i_k}\>\right] \\
	& = (1 - \theta_k) \E [\phi_k(\lambdad)] + \theta_k \E \left[ L(\y^{(k)}) +  \< \nabla L(\y^{(k)}), \lambdad-  \y^{(k)}\>\right] \\
	& \leq (1 - \theta_k) ( (1 - \delta_k) L(\lambdad) + \delta_k \phi_0(\lambdad) ) + \theta_k L(\lambdad) \\
	& = (1 - \delta_{k + 1} ) L(\lambdad) +  \delta_{k + 1} \phi_0(\lambdad)
	\end{align*}
	The first line uses the definition of $\phi_{k + 1}$ and the second line uses the law of total expectation and the fact that $(e_k, i_k)$ is sampled uniformly. The inequality leverages the inductive hypothesis and convexity of $L$. From \citet[\S 2]{nesterov2018lectures}, we know that the definition of $\delta_k$ from $\theta_k$ ensures that $\delta_k \stackrel{k}{\rightarrow} 0$. Therefore, $\{\phi_k, \delta_k\}_{k = 0}^K$ is a random estimate sequence.

As noted, the identities are fairly standard \citep{nesterov2018lectures,lee2013efficient}. We prove each claim in order. 
\begin{itemize}
    \item From definition of $\phi_{k + 1}$, computing the second derivative of the combination shows that it is constant at $(1 - \theta_k) \gamma_k$.
    
    \item Computing the gradient with respect to block-coordinate $\lambdad_{e_k, i_k}$ of the combination shows, at optimality, we have
    \begin{align*}
        (1 - \theta_k) \gamma_k (\lambdad_{e_k, i_k}- \v^{(k)}_{e_k, i_k}) - q\theta_k \nu_{e_k, i_k}^{\y^{(k)}} = 0    
    \end{align*}
    which implies
    \begin{align*}
        \lambdad_{e_k, i_k} = \v^{(k)}_{e_k, i_k} + \frac{q\theta_k}{\gamma_{k + 1}} \nu_{e_k, i_k}^{\y^{(k)}}
    \end{align*}
    For any other block-coordinate $(e, i)$, the optimality condition simply implies $\lambdad_{e, i} = \v_{e, i}^{(k)}$.

    \item The last claim can be show by inserting the minimizer, $\v_{e_k, i_k}^{(k)}$, into $\phi_{k + 1}$. Therefore, we have
    \begin{align*}
        \omega_{k + 1} & := \min_{\lambdad} \phi_{k + 1} \\
                        & = \phi_{k + 1}(\v^{(k + 1)}) \\
                        & = (1 - \theta_k) \omega_k + \frac{\gamma_{k + 1}}{2}\| \v^{(k)} - \v^{(k + 1)}\|^2_2  + \theta_k L(\y^{(k)}) - \theta_k q \<\nu_{e_k, i_k}^{\y^{(k)}}, \v^{(k + 1)}_{e_k, i_k} - \y^{(k)}_{e_k, i_k} \> \\
                        & = (1 - \theta_k) \omega_k + \frac{(q\theta_k)^2}{2\gamma_{k  +1}}\| \nu_{e_k, i_k}^{\y^{(k)}}\|^2_2  + \theta_k L(\y^{(k)}) - \theta_k q \<\nu_{e_k, i_k}^{\y^{(k)}}, \v^{(k )}_{e_k, i_k} + \frac{q\theta_k}{\gamma_{k + 1} }\nu_{e_k, i_k}^{\y^{(k)}} - \y^{(k)}_{e_k, i_k} \> \\
                        & = (1 - \theta_k) \omega_k + \theta_k L (\y^{(k)}) - \frac{(q\theta_k)^2}{2\gamma_{k + 1} } \| \nu_{e_k, i_k}^{\y^{(k)}}\|^2_2 - q\theta_k \< \nu_{e_k, i_k}^{\y^{(k)}}, \v^{(k )}_{e_k, i_k} - \y^{(k)}_{e_k, i_k} \>
    \end{align*}
\end{itemize}
\end{proof}

\subsection{Proof of $\LL_2$ Projection Lemma~\ref{lem::proj} and $\LL_2^{\lambdad}$ Projection  Lemma~\ref{lem::proj-slack}}\label{sec::approx-bound-lemma-proofs}


\begin{restatable}{lemma}{lemproj}
\label{lem::proj}
 For $\lambdad \in \R^{r_D}$ and $\mu^{\lambdad} \in \LL_{2}^{\nu^{\lambdad}}$, Algorithm~\ref{alg::proj}  returns a point $\widehat \mu = \emph \proj(\mu^{\lambdad}, 0)$ such that $\widehat \mu_i = \mu_i^{\lambdad}$ for all $i \in V$ and
\begin{align*}
    \sum_{e \in E} \| \mu^{\lambdad}_e - \widehat \mu_e \|_1 \leq 2 \sum_{e \in E, i \in e}  \| \nu_{e, i}^{\lambdad} \|_1.
\end{align*}
\end{restatable}

\begin{proof}
    Since $\nu = 0$, we know that $\mu_i^{\lambdad} + \nu_{e, i} = \mu_i^{\lambdad} \in \Delta_n$ for all $i \in V$ and $e \in N_i$. For any $(i,j) = e \in E$, \proj{} applies Algorithm~2 of \citet{altschuler2017near} to generate $\widehat \mu_e$ from $\mu^{\lambdad}_e$ with the following guarantee due to \citet[Lemma 7]{altschuler2017near}:
    \begin{align*}
        \| \widehat \mu_e - \mu_e^{\lambdad} \|_1 \leq 2 \| S_{e, i}^{\lambdad} -  \mu_i^{\lambdad}\|_1 + 2 \| S_{e, j}^{\lambdad} - \mu_j^{\lambdad}\|_1 
    \end{align*}
    and $\widehat \mu_e \in \U_d (\mu_i^{\lambdad}, \mu_j^{\lambdad})$. Applying this guarantee for all edges in $E$ gives the result. 
\end{proof}

\begin{restatable}{lemma}{lemprojslack}
\label{lem::proj-slack}
Let $\mu \in \LL_2$ and $\lambdad \in \R^{r_D}$. Define $\delta = \max_{e \in E, i \in e} \|\nu_{e, i}^{\lambdad}\|_1$ There exists $\widehat \mu$ in the slack polytope $\LL_2^{\nu^{\lambdad}}$ such that
\begin{align*}
    \| \mu - \widehat \mu \|_1 \leq  16 (m + n) d \delta + 2 \sum_{e \in E, i \in e}  \| \nu_{e, i}^{\lambdad} \|_1
\end{align*}
\end{restatable}

\begin{proof}
    For convenience, we just write $\nu$ for the slack, dropping the notational dependence on $\lambdad$. We will proceed with this proof by constructing such a $\widehat \mu$ in two cases. We would like to show that the edge marginals $\mu_e$ can be modified to give $\widehat \mu \in \LL_2^{\nu}$. To do this, we aim to use Algorithm~\ref{alg::proj} to match $\widehat \mu_e$ to the modified marginals $\mu_i + \nu_{e,i}$ for every $e \in E$ and $i \in e$. As long as $\mu_i + \nu_{e, i} \in \Delta_d$, setting $ \mu_i' = \mu_i$ and $ \mu_e' = \mu_e$ and computing $\widehat \mu = \proj(\widetilde \mu, \nu )$ would return $\widehat \mu \in \LL_{2}^{\nu}$ that satisfies the condition by Lemma~\ref{lem::proj}.
    
    However, if $\mu_i + \nu_{e, i} \not \in \Delta_d$, then $\exists  \ x \in \chi$ such that $\mu_i(x) + \nu_{e, i}(x) \not \in [0, 1]$. Consider the case where $\delta \leq \frac{1}{2d}$. We aim to create a temporary marginal vector $\mu'$ which is made by  modifying $\mu_i$ appropriately until the slack can be added to $\mu'_i$ while maintaining a valid distribution. To do this, we set $ \mu_i'$ as the convex combination with the uniform distribution
    \begin{align*}
         \mu_i' = (1 - \theta) \mu_i + \theta \unif(\chi)
    \end{align*}
    for some $\theta \in [0, 1]$. Choosing $\theta = d\delta$ ensures that \begin{align*}
        \delta \leq  \mu'(x) \leq 1 - \delta \quad \forall \ x\in \chi,
    \end{align*}
    Furthermore, $ \mu_i' \in \Delta_d$ because $\Delta_d$ is convex and  we have\begin{align*}
    \| \mu_i' - \mu_i \|_1 & = \sum_{x \in \chi} \delta | 1 - d  \mu_{i}(x) | \\
    & \leq \sum_x  \delta + \delta d \mu_i(x) \\
    & = 2d\delta
    \end{align*} Then we set $\mu'_e = \mu_e$ for all $e\in E$. Using Algorithm~\ref{alg::proj}, we compute $\widehat \mu = \proj(\mu', \nu) \in \LL_2^{\nu}$. Together with Lemma~\ref{lem::proj}, we have that
    \begin{align*}
        \| \widehat \mu - \mu\|_1 & = \sum_{i \in V} \| \widehat \mu_i - \mu_i\|_1 + \sum_{e \in E} \| \widehat \mu_e - \mu_e\|_1 \\
        & \leq  2 n d \delta + 2 \sum_{e \in E, i \in e} \| \nu_{e, i}\|_1 + \|  \mu_i' - \mu_i\|_1 \\
        & \leq (n + 8m) d \delta + 2 \sum_{e \in E, i \in e} \| \nu_{e, i}\|_1 
    \end{align*}
    On the other hand, consider the case where $\delta > \frac{1}{2d}$. Then instead we choose the temporary marginal vector as $\mu_i' = \mu^{\lambdad}_i$ for all $i \in V$ and $\mu_e' = \mu_e$ for all $e \in E$, which ensures that $\mu_i' + \nu_{e, i} \in \Delta_d$ by definition of $\nu$. We then compute $\widehat \mu = \proj(\mu', \nu )$, which ensures
    \begin{align*}
        \|\widehat \mu - \mu\|_1 & \leq \sum_{i \in V} \| \mu_i  -\mu_i^{\lambdad} \|_1 + 2 \sum_{e \in E, i \in e} \| \mu_i  -\mu_i^{\lambdad} \|_1 + \| \nu_{e, i}\|_1 \\
        & \leq 2n  + 8m +  2 \sum_{e \in E, i \in e} \| \nu_{e, i}\|_1  \\
        & \leq 4n d \delta + 16m d \delta + 2 \sum_{e \in E, i \in e} \| \nu_{e, i}\|_1
    \end{align*}
    where the second inequality uses the fact that the $l_1$ distance is bounded by $2$ and the last inequality uses the assumption that $\delta > \frac{1}{2d}$. We take the worst of these two cases for the final result.
\end{proof}

\subsection{Proof of Proposition~\ref{prop::approx-bound}}

\propapproxbound*

\begin{proof}
	Consider $\mu^{\lambdad}$, which may not lie in $\LL_2$. It does, however, lie within its own slack polytope $\LL_2^{ \nu^{\lambdad}}$ from Definition~\ref{def::slack}. Therefore, it can be seen that $\mu^{ \lambdad}$ is a solution to
	\begin{align}\label{eq::slack-obj}
	\minimize \quad \<C, \mu\> - \frac{1}{\eta} H(\mu) \quad \st \quad \mu \in \LL_{2}^{ \lambdad}
	\end{align}
	Then, consider the point $\widehat \mu = \proj(\mu^{ \lambdad}, 0) \in \LL_2$. Let $\mu^* \in \argmin_{\mu \in \LL_2} \<C, \mu\>$. We have
	\begin{align}\begin{split} \label{eq::primal-bounds-1}
	\< C, \widehat \mu - \mu^*\> & = \< C,  \widehat \mu -\mu^{  \lambdad} + \mu^{  \lambdad} - \mu^* \> \\
	& \leq \|C\|_\infty \| \widehat \mu - \mu^{ \lambdad}\|_1 + \< C, \mu^{  \lambdad} - \mu^*\>.
	\end{split} \end{align}
	Note that the last term in the right-hand side can be written as
	\begin{align} \begin{split} \label{eq::primal-bounds-2}
	\< C, \mu^{  \lambdad} - \mu^*\> & = \< C, \mu^{  \lambdad} - \widehat \mu^* + \widehat \mu^* - \mu^*\> \\
	& \leq  \| C\|_\infty \| \widehat \mu^* - \mu^*\|_1 + \< C, \mu^{ \lambdad} - \widehat \mu^*\>,
	\end{split} \end{align}
	where $\widehat \mu^* \in \LL_{2}^{\nu^{\lambdad}}$ is the existing vector from Lemma~\ref{lem::proj-slack} using  $\mu^* \in \LL_2$ and slack from $ \lambdad$. Because $\mu^{ \lambdad}$ is the solution to (\ref{eq::slack-obj}), we further have
	\begin{align}\begin{split} \label{eq::primal-bounds-3}
	\<C, \mu^{ \lambdad} - \widehat \mu^* \> & \leq \frac{1}{\eta} (H(\mu^{ \lambdad}) - H(\widehat \mu^*)) \\
	& \leq \frac{n \log d + 2 m \log d}{\eta}
	\end{split} \end{align}
	Combining inequalities (\ref{eq::primal-bounds-1}), (\ref{eq::primal-bounds-2}), and (\ref{eq::primal-bounds-3}) shows that
	\begin{align*}
	\< C, \widehat \mu-\mu^*\> & \leq \|C\|_\infty ( \| \widehat \mu - \mu^{ \lambdad} \|_1 + \|\widehat\mu^* - \mu^*\|_1 )
	+ \frac{n \log d + 2 m \log d}{\eta}
	\end{align*}
	Using Lemma~\ref{lem::proj} and \ref{lem::proj-slack}, we can further bound this as
	\begin{align*}
	\< C, \widehat \mu-\mu^*\> & \leq \| C\|_\infty  \left( 16( m + n) d \delta + \sum_{e, i}  4\| \nu^{ \lambdad}_{e, i}\|_1 \right)
	+ \frac{n \log d + 2 m \log d}{\eta}.
	\end{align*}    
\end{proof}

\subsection{Proof of $G(\eta)$ Upper Bound}
In the proof of Lemma~\ref{lem::acc-dual-conv}, we used the fact that the numerator of the final convergence rate can be bounded by $G(\eta)^2$. Here, we formally state this result and prove it.
\begin{restatable}{lemma}{lemgbound}
	\label{lem::g-bound}
	It holds that \begin{align*}
	4 L(0) - 4 L(\lambdad^*)  + 16m^2 \eta \| \lambdad^*\|^2_2 \leq G(\eta)^2,
	\end{align*}
	where $G(\eta) := 24 md (m + n) ( \sqrt \eta \|C\|_\infty  + \frac{\log d}{\sqrt \eta} )$. 
\end{restatable}

The proof requires bounding both $L(0) - L(\lambdad^*)$, which we have already done in Lemma~\ref{lem::function-value}, and bounding the norm $\| \lambdad^*\|_2^2$. We rely on the following result from \citet{meshi2012convergence}.
\begin{lemma}\label{lem::dual-norm}
 There exists $\lambdad^* \in \Lambda^*$ such that
 \begin{align*}
     \|\lambdad^* \|_1 & \leq 2 d (n + m) \| C\|_\infty + \frac{4 d (m + n)}{\eta} \log d \\
     & \leq \frac{4 d( m + n)}{\eta} ( \eta \|C\|_\infty + \log d)
 \end{align*}
\end{lemma}
\begin{proof}
Modifying \citet[Supplement Lemma 1.2]{meshi2012convergence} for our definition of $H$ gives us
\begin{align*}
    \| \lambdad^* \|_1 \leq  2d ( L(0) - n - m - \< C, \mu^{\lambdad^*}\> + \frac{1}{\eta} H(\mu^{\lambdad^*})).
\end{align*}
Using Cauchy-Schwarz and maximizing over the entropy yields the result.
\end{proof}

\begin{proof}[Proof of Lemma~\ref{lem::g-bound}]
	Using these results, we can prove the claim. We bound the square root of the numerator, multiplied by $\sqrt \eta$:
	\begin{align*}
	\sqrt{\eta  \left(4 L(0) - 4 L(\lambdad^*)  + 16m^2 \eta \| \lambdad^*\|^2_2 \right)} & \leq \sqrt{ 8 ( m + n) ( \eta \| C\|_\infty  + \log d) + 16 m^2  \left( 4d(m + n) \left( \eta \| C\|_\infty + \log d \right)\right)^2 } \\
	& \leq \sqrt{ 8 (m + n) (\eta \| C\|_\infty + \log d)  } + 16 m d (m + n) ( \eta \| C\|_\infty + \log d)  \\
	& \leq 24 m d (m + n) ( \eta \| C\|_\infty + \log d)
	\end{align*}
	The first inequality used Lemma~\ref{lem::function-value} and Lemma~\ref{lem::dual-norm}. The second inequality uses the fact that $\sqrt{a + b} \leq \sqrt{a} + \sqrt{b}$ for $a, b \geq 0$. The last inequality uses the fact that the first term is greater than 1 under the assumption $d \geq 2$. Dividing through by $\sqrt{\eta}$ gives the result.
\end{proof}

\section{Proof of Theorem~\ref{th::standard}}\label{sec::classic-proof}
We begin with a complete proof of Theorem~\ref{th::standard} for \emp{} and then show how to modify it slightly for \smp{}. The result also builds on some of the same technical lemmas used in the proof of Theorem~\ref{th::accel}.

\subsection{Edge Message Passing}
The cornerstone of the proof is showing that the expected slack norms can be bounded over iterations.

\begin{restatable}{lemma}{lemstandardslack}
\label{lem::standard-slack}
Let $L^* = \min_{\lambdad} L(\lambdad)$ and let $\widehat \lambdad$ be the output of Algorithm~\ref{alg::semp} after $K$ iterations with \emph{\emp{}} and a uniform distribution. For any $e \in E$ and $i \in e$, the expected norm of the constraint violation in $\LL_2$ is bounded as
\begin{align*}
    \E  \sum_{e \in E, i \in e} \| S_{e, i}^{\widehat \lambdad} - \mu_i^{\widehat \lambdad} \|_1^2  \leq \frac{8m \eta (L(0) - L^*)}{K}
    \end{align*}
    
\end{restatable}

\begin{proof}
    From Lemma~\ref{lem::emp-improvement}, we have that the expected improvement is lower bounded at each iteration
    \begin{align*}
        \E \left[ L(\lambdad^{(k)}) - L(\lambdad^{(k + 1)}) \right] \geq \frac{1}{4\eta} \E \left[ \| \nabla_{e_k, i_k} L(\lambdad^{(k)}) \|_1^2 \right],
    \end{align*}
    Then, using that $\nabla_{e, i}  L(\lambdad) = \mu_i^{\lambdad} - S_{e, i}^{\lambdad}$, we apply the bound $k = 1, 2, \ldots, K$:
    \begin{align*}
         L(0) - L^*  & \geq \frac{1}{4\eta}\sum_{k = 0}^{K - 1} \E \left[ \| S_{e_k, i_k}^{ \lambdad^{(k)}} - \mu_{i_k}^{ \lambdad^{(k)}} \|_1^2 \right] \\
         & =   \frac{1}{8 m \eta} \sum_{k = 0}^{K - 1} \sum_{e \in E, i \in e} \E \left[  \|S_{e, i}^{ \lambdad^{(k)}} - \mu_i^{\lambdad^{(k)}} \|_1^2 \right] \\
         & \geq  \frac{K}{8m \eta}  \sum_{e \in E, i \in e} \E \left[ \| S_{e, i}^{\widehat \lambdad} - \mu_i^{\widehat \lambdad} \|_1^2 \right] 
         \\
    \end{align*}
    The equality uses the law of total expectation, conditioning on $\lambdad^{(k)}$. The second inequality uses the fact that $\widehat \lambdad$ is chosen to minimize the sum of squared constraint violations.
    
\end{proof}

Next, we provide a bound on the initial function value gap.
\begin{restatable}{lemma}{lemfunctionvalue}
\label{lem::function-value}
For $\lambdad^* \in \Lambda^*$ it holds that $L(0) - L(\lambdad^*) \leq 2(m + n) \| C\|_\infty  + \frac{2}{\eta} (m + n) \log d$.
\end{restatable}

\begin{proof}
    We will bound both $L(0)$ and $L(\lambdad^*)$ individually. First, from the definition
    \begin{align*}
        L(0) & = \frac{1}{\eta} \sum_{i} \log \sum_{x_i} \exp( - \eta C_i(x_i) ) 
             +  \frac{1}{\eta}\sum_{e} \log \sum_{x_e} \exp(- \eta C_e(x_e))  \\
            & \leq \frac{1}{\eta} \sum_{i} \log \sum_{x_i} \exp( \eta \|C\|_\infty ) 
             + \frac{1}{\eta}\sum_{e} \log \sum_{x_e} \exp(\eta \|C\|_\infty ) \\
            & = (n + m)\|C\|_\infty + \frac{n}{\eta} \log d  + \frac{2m}{\eta}\log d.
    \end{align*}
    For $L(\lambdad^*)$, we recognize that $L$ is simply the negative of the primal problem (\ref{eq::primal}), shifted by a constant amount. In particular, we have
    \begin{align*}
        - L(\lambdad^*) - \frac{1}{\eta} (n + m) = \< C, \mu^* \> - \frac{1}{\eta} H(\mu^*)
    \end{align*}
    For some $\mu^*$ that solves (\ref{eq::primal}). Note that $H$ is offset with a linear term (different from the usual definition of the entropy) that exactly cancels the $-\frac{1}{\eta}(n + m)$ on the left-hand side. We then conclude $-L(\lambdad^*)  \leq (m + n)\| C\|_\infty$ by Cauchy-Schwarz. Summing these two gives the desired result.
\end{proof}

\begin{proof}[Proof of Theorem~\ref{th::standard} for \emph{\emp{}}]
Fix $\epsilon' > 0$. Lemma~\ref{lem::standard-slack} and Lemma~\ref{lem::function-value} ensure that 
\begin{align*}
    \E \sum_{e \in E, i \in e} \| \nu_{e, i}^{\widehat\lambdad}  \|_1^2 \leq (\epsilon')^2
\end{align*}
after
\begin{align}\label{eq::standard-eps-prime}
    K = \frac{16 m (m + n) ( \eta\|C\|_\infty +  \log d)  }{(\epsilon')^2}
\end{align}
iterations. By Jensen's inequality and recognizing that the norms are non-negative, this also implies that
\begin{align*}
    \E [ \| \nu_{e, i}^{\widehat\lambdad}  \|_1] \leq \epsilon' \quad \forall e \in E, i \in e
\end{align*}
after the same number of iterations.

Now, we use the upper bound due to the approximation from Proposition~\ref{prop::approx-bound} and take the expectation, giving
\begin{align*}
    \E \left[ \< C, \widehat \mu-\mu^*\> \right] & \leq \| C\|_\infty \left( 8 m \epsilon' + 16(m + n)d\E[ \delta]  \right) \\
    & + \frac{n \log d + 2 m \log d}{\eta},
\end{align*}
where
\begin{align*}
    \E \left[ \delta \right]^2  \leq  \E [ \delta^2 ]  \leq \E  \sum_{e \in E, i \in e} \| \nu_{e, i}^{\widehat \lambdad}\|_1^2   
                     \leq (\epsilon')^2.
\end{align*}
Therefore, the bound becomes
\begin{align*}
     \E \left[  \< C, \widehat \mu-\mu^*\> \right] & \leq 24 (m + n)d\|C\|_\infty \epsilon'  \\
     &\quad + \frac{n \log d + 2 m \log d}{\eta}
\end{align*}
We conclude the result from substituting into (\ref{eq::standard-eps-prime}), using the definition of $\eta$ and choosing $\epsilon' = \frac{\epsilon}{48 (m + n)d \|C\|_\infty}$
\end{proof}

\subsection{Star Message Passing}

The significant difference between the $\smp{}$ proof and the \emp{} proof is that there is variable improvement at each update, dependent on the degree of the node being updated. Using the distribution from (\ref{eq::non-uniform}), we ensure that the improvement becomes uniform in expectation. This analysis is similar to weighting coordinates by their coordinate-wise smoothness coefficients in coordinate gradient algorithms \citep{nesterov2012efficiency}.

A slight modification of the proof of \cite{meshi2012convergence} is required the get the tighter $l_1$-norm lower bound.

\lemsmpimprov*
\begin{proof}
\citet{meshi2012convergence} show that
\begin{align*}
    L(\lambdad) - L(\lambdad') = -\frac{1}{\eta} \log \left( \sum_{x_i}  \left( \mu_{i}^{\lambdad}  \prod_{e \in N_i} S_{e, i}^{\lambdad}(x_i) \right)^{\frac{1}{|N_i| + 1}}\right)^{|N_i| + 1},
\end{align*}
and further
\begin{align*}
    |N_i| - |N_i|   \left( \sum_{x_i}  \left( \mu_{i}^{\lambdad}  \prod_{e \in N_i} S_{e, i}^{\lambdad}(x_i) \right)^{\frac{1}{|N_i| + 1}}\right)^{|N_i| + 1}  & \geq  \sum_{e \in N_i} \left(1 - \left( \sum_{x_i} \sqrt{\mu_i^{\lambdad} (x_i) S^{\lambdad}_{e,i} (x_i)} \right)^2 \right) 
\end{align*}
We recognize the inner term of the square as the Bhattacharyya coefficient which satisfies $BC \in [0, 1]$. Therefore,
\begin{align*}
    \sum_{e \in N_i} \left(1 - \left( \sum_{x_i} \sqrt{\mu_i^{\lambdad} (x_i) S^{\lambdad}_{e,i} (x_i)} \right)^2 \right)  & \geq \sum_{e \in N_i} \left(1 - \left( \sum_{x_i} \sqrt{\mu_i^{\lambdad} (x_i) S^{\lambdad}_{e,i} (x_i)} \right)\right) \\
    & = \sum_{e \in N_i}  h^2(\mu_i^{\lambdad}, S^{\lambdad}_{e, i})
\end{align*}
Then,
\begin{align*}
    \left( \sum_{x_i}  \left( \mu_{i}^{\lambdad}  \prod_{e \in N_i} S_{e, i}^{\lambdad}(x_i) \right)^{\frac{1}{|N_i| + 1}}\right)^{|N_i| + 1} & \leq 1 - \frac{1}{N_i} \sum_{e \in N_i}  h^2(\mu_i^{\lambdad}, S^{\lambdad}_{e, i})
\end{align*}
Finally, we lower bound the original difference of values
\begin{align*}
    L(\lambdad) - L(\lambdad')  & \geq - \frac{1}{\eta} \log \left( 1 - \frac{1}{N_i} \sum_{e \in N_i}  h^2(\mu_i^{\lambdad}, S^{\lambdad}_{e, i}) \right) \\
    & \geq \frac{1}{N_i\eta } \sum_{e \in N_i}  h^2(\mu_i^{\lambdad}, S^{\lambdad}_{e, i}) \\
    & \geq \frac{1}{8N_i\eta } \sum_{e \in N_i}  \| S^{\lambdad}_{e, i} - \mu_i^{\lambdad}\|_1^2
\end{align*}
\end{proof}

\begin{lemma}\label{lem::standard-slack-smp}
Let $L^* = \min_{\lambdad} L(\lambdad)$ and let $\widehat \lambdad$ be the output of Algorithm~\ref{alg::semp} after $K$ iterations with \emph{\smp{}} and distribution (\ref{eq::non-uniform}). Define $N = \sum_{j \in V} |N_j|$. For any $e \in E$ and $i \in e$, the expected norm of the constraint violation in $\LL_2$ is bounded as
\begin{align*}
    \E  \sum_{e \in E, i \in e} \| S_{e, i}^{\widehat \lambdad} - \mu_i^{\widehat \lambdad} \|_1^2  \leq \frac{8 N  \eta (L(0) - L^*)}{K}
    \end{align*}
    
\end{lemma}
\begin{proof}
	Lemma~\ref{lem::smp-improvement} gave us the following lower bound on the improvement:
\begin{align*}
 \E \left[ L(\lambdad^{(k)}) - L(\lambdad^{(k + 1)}) \right]
    \geq \E \left[  \frac{1}{8|N_{i_k}|\eta} \sum_{e \in N_{i_k}}  \|\nu_{e, i_k}^{\lambdad^{(k)}}\|_1 ^2 \right]
\end{align*}
Then, since $i_k$ is chosen with probability $p_i = \frac{|N_i|}{N}$, we can apply the bound for $k = 1, 2, \ldots, K$ and expand the expectations:
\begin{align*}
L(0) - L^* & \geq \sum_{k = 0}^{K - 1}  \E \left[  \frac{1}{8|N_{i_k}|\eta} \sum_{e \in N_{i_k}}  \|\nu_{e, i_k}^{\lambdad^{(k)}}\|_1 ^2 \right] \\
& = \frac{1}{8 N \eta}\sum_{k = 0}^{K - 1} \E  \sum_{e \in E, i \in e}\|\nu_{e, i}^{\lambdad^{(k)}}\|_1 ^2  \\
& \geq \frac{1}{8 N \eta}\sum_{k = 0}^{K - 1}   \sum_{e \in E, i \in e}\E  \left[ \|\nu_{e, i}^{\widehat \lambdad}\|_1 ^2  \right]
\end{align*}
The equality uses the law of total expectation, conditioning on $\lambdad^{(k)}$. The second inequality uses the fact that $\widehat \lambdad$ is chosen to minimize the sum of squared constraint violations.
%
\end{proof}

The rest of the proof for \smp{} proceeds in an identical manner to the case for \emp{}; however, we simply replace the $8m\eta$ with $8N\eta$ everywhere. This stems from the fact that we can now guarantee
\begin{align*}
      \E  \sum_{e \in E, i \in e} \| S_{e, i}^{\widehat \lambdad} - \mu_i^{\widehat \lambdad} \|_1^2  \leq (\epsilon')^2
\end{align*}
in $\frac{8 N  \eta (L(0) - L^*)}{(\epsilon')^2}$ iterations instead. We can then use the same upper bound from Lemma~\ref{lem::function-value} and substitute in the same choices of $\epsilon'$ and $\eta$ as in \emp{} to get the result.
\section{Proof of Theorem~\ref{th::accel} for \smp{}}

The proof for \smp{} essentially follows the same structure, but it requires defining the estimate sequence in slightly different way. Define the probability distribution $\{ p_i\}_{i \in V}$ over $V$ with $p_i = \frac{|N_i|}{\sum_{j \in V} | N_j|}$. We propose the candidate:
\begin{align}\begin{split} \label{eq::estimate-sequence-smp}
 i_k & \sim \text{Cat}( V, \{ p_i\}_{i \in V} ) \\
\delta_{k + 1} & = (1 - \theta_k) \delta_k \\
\phi_{k + 1} (\lambdad) & = (1 - \theta_k) \phi_k(\lambdad) + \theta_k L(\y^{(k)})  -  \frac{\theta_k}{p_{i_k}} \< \nu_{\cdot, i_k}^{\y^{(k)}}, \lambdad_{\cdot, i_k} - \y^{(k)}_{\cdot, i_k}\>
\end{split}
\end{align}

Then, we show that this is indeed an estimate sequence with a conducive structure.

\begin{restatable}{lemma}{lemestseqsmp}
	\label{lem::est-seq-smp}
	The sequence $\{\phi_k, \delta_k\}_{k = 0}^K$ defined in (\ref{eq::estimate-sequence-smp}) is a random estimate sequence. Furthermore, it maintains the form $\phi_k(\lambdad) = \omega_k + \frac{\gamma_k}{2}\| \lambdad - \v^{(k)} \|$ for all $k$ where
	\begin{align*}
	\gamma_{k + 1} & = (1 - \theta_k) \gamma_k \\
	\v^{(k + 1)}_{\cdot, i} & = \begin{cases} 
	\v_{\cdot, i}^{(k)} + \frac{\theta_k}{p_i\gamma_{k + 1}} \nu_{\cdot, i}^{\y^{(k)}} & \text{if }  i =  i_k \\
	\v_{\cdot, i}^{(k)} &  \emph{\text{otherwise}}
	\end{cases}\\
	\omega_{k + 1} & = (1 - \theta_k) \omega_k + \theta_k L(\y^{(k)})  - \frac{\theta_k^2}{2\gamma_{k+1}p_{i_k}^2} \| \nu_{\cdot, i_k}^{\y^{(k)}} \|_2^2   - \frac{\theta_k}{p_{i_k}}  \<\nu_{\cdot, i_k}^{\y^{(k)}}, \v^{(k)}_{\cdot, i_k} - \y^{(k)}_{\cdot, i_k} \> 
	\end{align*}
\end{restatable}

\begin{proof}
	To show that this is an estimate sequence, the proof is essentially identical to the \emp{} case. The only exception is that we take expectation over $V$ with distribution $\{ p_i\}_{i \in V}$. However, this ensures that
	\begin{align*}
	\E [  \frac{\theta_k }{p_{i_k}}\< \nu_{\cdot, i_k}^{\y^{(k)}}, \lambdad_{\cdot, i_k} - \y^{(k)}_{\cdot, i_k}\>  ] = \theta_k \E [ \< \nabla L(\y^{(k)}), \lambdad - \y^{(k)}\> ]
	\end{align*}
	by the law of total expectation. So the the proof that this is an estimate sequence remains the same.
	
	To show that it retains the desired quadratic structure, we again analyze all terms of interest
	\begin{itemize}
		\item $\gamma_{k + 1}$ is identical to the \emp{} case so the result holds.
		\item Taking the gradient with respect to $\lambdad_{\cdot, i}$, we have that the optimality conditions, for $i = i_k$, are
		\begin{align*}
			\gamma_{k + 1} ( \lambdad_{\cdot, i_k} - \v^{(k)}_{\cdot, i_k}) - \frac{\theta_k }{p_{i_k}} \nu_{\cdot, i_k}^{\y^{(k)}}  = 0
		\end{align*}
		and, for all other $i$, they are
		\begin{align*}
		\gamma_{k + 1} ( \lambdad_{\cdot, i_k} - \v^{(k)}_{\cdot, i_k}) = 0.
		\end{align*}
		These conditions imply the given construction for $\v^{(k + 1)}$.
		\item We can then compute $\omega_{k+ 1}$ by plugging in the choice for $\v^{(k + 1)}$ again:
		\begin{align*}
		\omega_{k + 1} & := \min_{\lambdad} \phi_{k + 1} \\
		& = \phi_{k + 1}(\v^{(k + 1)}) \\
		& = (1 - \theta_k) \omega_k + \frac{\gamma_{k + 1}}{2}\| \v^{(k)} - \v^{(k + 1)}\|^2_2  + \theta_k L(\y^{(k)}) - \frac{\theta_k }{p_{i_k}} \<\nu_{\cdot, i_k}^{\y^{(k)}}, \v^{(k + 1)}_{\cdot, i_k} - \y^{(k)}_{\cdot, i_k} \> \\
		& = (1 - \theta_k) \omega_k + \frac{\theta_k^2}{2\gamma_{k  +1}p_{i_k}^2}\| \nu_{\cdot, i_k}^{\y^{(k)}}\|^2_2  + \theta_k L(\y^{(k)}) - \frac{\theta_k }{p_{i_k}} \<\nu_{\cdot, i_k}^{\y^{(k)}}, \v^{(k )}_{\cdot, i_k} + \frac{\theta_k }{\gamma_{k + 1} p_{i_k}}\nu_{\cdot, i_k}^{\y^{(k)}} - \y^{(k)}_{\cdot, i_k} \> \\
		& = (1 - \theta_k) \omega_k + \theta_k L (\y^{(k)}) - \frac{\theta_k^2}{2\gamma_{k  +1}p_{i_k}^2} \| \nu_{\cdot, i_k}^{\y^{(k)}}\|^2_2 - \frac{\theta_k}{p_{i_k}} \< \nu_{\cdot, i_k}^{\y^{(k)}}, \v^{(k )}_{\cdot, i_k} - \y^{(k)}_{\cdot, i_k} \>
		\end{align*}
	\end{itemize}
\end{proof}

We now provide a faster convergence guarantee on the dual objective function for \smp{} which depends on $N = \sum_{j \in V} |N_j|$.

\begin{lemma}\label{lem::acc-dual-conv-smp}
    For the random estimate sequence in (\ref{eq::estimate-sequence-smp}), let $\{ \lambdad^{(k)}\}_{k = 0}^K$ and $\{\y^{(k)}\}_{k = 0}^K$ be defined as in Algorithm~\ref{alg::accel-smp} with $\lambdad^{(0)} = 0$. Then, the dual objective error in expectation can be bounded as 
    \begin{align*}
        \E [ L(\lambdad^{(k)}) - L(\lambdad^*)] \leq \frac{G_\smp(\eta)^2}{(k + 2)^2},
    \end{align*}
   	where $G_\smp(\eta) := 24 Nd (m + n) ( \sqrt \eta \|C\|_\infty  + \frac{\log d}{\sqrt \eta} )$ and $N = \sum_{j \in V} |N_j|$.
\end{lemma}
\begin{proof}
	As in the \emp{} proof, it suffices to show that $\E[\omega_{k+ 1}] \geq \E[ L(\lambdad^{(k + 1)})]$ by induction. As before we have
	\begin{align*}
	\E [ \omega_{k + 1} ] 
	& \geq  (1 - \theta_k) \E [L(\lambdad^{(k)})] + \theta_k\E [ L(\y^{(k)})]  - \E \left[  \frac{\theta_k^2}{2\gamma_{k+1}p_{i_k}^2} \| \nu_{\cdot, i_k}^{\y^{(k)}} \|_2^2   - \frac{\theta_k}{p_{i_k}} \<\nu_{\cdot, i_k}^{y^{(k)}}, \v^{(k)} - \y^{(k)} \> \right]  \\
	& \geq \E \left[ L(\y^{(k)}) - \frac{\theta_k ^2}{2\gamma_{k+1}p_{i_k}^2} \| \nu_{\cdot, i_k}^{\y^{(k)}} \|_2^2\right]  
	+ (1 - \theta_k) \E \left[ \< \nabla L(\y^{(k)}), \lambdad^{(k)} - \y^{(k)}   \> \right] 
	 +  \theta_k \E \left[ \< \nabla L(\y^{(k)}), \v^{(k)} - \y^{(k)}   \> \right] \\
	& = \E \left[ L(\y^{(k)}) - \sum_{i \in V} \frac{\theta_k^2 }{2\gamma_{k+1}p_{i}} \| \nu_{\cdot, i}^{\y^{(k)}} \|_2^2 \right],
	\end{align*}
	where the last line comes from the definition of $\y^{(k)}$.
	Choosing $\theta_k$ such that $\theta_k^2 = \frac{\gamma_{k+1} \min_j |N_j|}{4 \eta N^2}$ results in
	\begin{align*}
	\E [ \omega_{k + 1} ]  & \geq \E \left[ L(\y^{(k)}) - \sum_{i \in V} \frac{1 }{8 N \eta} \| \nu_{\cdot, i}^{\y^{(k)}} \|_2^2 \right] \\
	& =\E \left[ L(\y^{(k)}) - \frac{1 }{8 N \eta} \| \nabla L(\y^{(k)}\|_2^2 \right]
	\end{align*}
	Recall, from the improvement in Lemma~\ref{lem::smp-improvement}, we have
	\begin{align*}
	 \E [ L(\lambdad^{(k + 1)}) ] & \leq \E [ L(\y^{(k)})]  -   \E \left[ \frac{1}{8  |N_{i_k}| \eta} \|  \nu_{\cdot, i_k}^{\y^{(k)}} \|_2^2 \right] \\
	 & = \E [ L(\y^{(k)})]  -   \E \left[ \frac{1}{8 N \eta} \|\nabla L(\y^{(k)}) \|_2^2 \right]
	\end{align*}
	Therefore, by this induction, the inequality $\E [ L(\lambdad^{(k + 1)}) ] \leq \E [ \omega_{k + 1} ] $ holds for all $k$. Furthermore, by choosing $\gamma_0 = \frac{4N^2 \eta}{\min_{j} | N_j|}$, we ensure that $\theta_k$ can be updated recursively as in Algorithm~\ref{alg::accel-smp} and the update equation for $\v$ is simplified to
	\begin{align*}
	    \v^{(k + 1)}_{\cdot, i} & = \begin{cases} 
	\v_{\cdot, i}^{(k)} + \frac{\min_j | N_j| }{4p_{i_k}\theta_k \eta N} \nu_{\cdot, i}^{\y^{(k)}} & \text{if }  i =  i_k \\
	\v_{\cdot, i}^{(k)} &  \emph{\text{otherwise}}
	\end{cases}.
	 \end{align*}
	 
	 Using the property of randomized estimate sequences derived in Section~\ref{sec::main-results}, we can bound the expected error in the dual norm as
	 \begin{align*}
	     \E (L(\lambdad^{(k)})] - L^* & \leq \frac{4}{(k + 2)^2} \left( L(0) - L^* + \frac{\gamma_0}{2}\|\lambdad\|_2^2 \right) \\
	     & = \frac{4}{(k + 2)^2} \left( L(0) - L^* + \frac{2N^2 \eta }{\min_j | N_j| }\|\lambdad^*\|_2^2 \right) \\
	     & \leq \frac{4}{(k + 2)^2} \left( L(0) - L^* + 2N^2 \eta \|\lambdad^*\|_2^2 \right)
	 \end{align*}
	 
	 The numerator can then be bounded in an identical manner to the \emp{} proof by replacing $4m^2$ with $2N^2$ in Lemma~\ref{lem::g-bound},  instead yielding $G_{\smp{}} (\eta) = 40 md (m + n) ( \sqrt \eta \|C\|_\infty  + \frac{\log d}{\sqrt \eta} )$, which is only different by a constant. We then have
	 \begin{align*}
	     \E (L(\lambdad^{(k)})] - L^* \leq \frac{G_{\smp}(\eta)}{(k + 2)^2}
	 \end{align*}
\end{proof}

With these tools, we are ready to present the proof of Theorem~\ref{th::accel} for \smp{}.

\begin{proof}[Proof of Theorem~\ref{th::accel} for \emph\smp{}].
Let $\widehat \lambdad$ be the output from Algorithm~\ref{alg::accel-smp} after $K$ iterations.
From Lemma~\ref{lem::smp-improvement}, we can lower bound the result in Lemma~\ref{lem::acc-dual-conv-smp} with
\begin{align*}
    \frac{1}{8\eta |N_i| }\E \left[ \sum_{e \in N_i} \| \nu_{e, i}^{\widehat \lambdad}\|_1^2 \right] & \leq \E [ L(\widehat \lambdad)] - L^* \\
    & \leq \frac{G_{\smp}(\eta) }{(K + 2)^2}
\end{align*}
for all $i \in V$. This further implies that
\begin{align*}
    \frac{1}{8\eta |N_i|} \E \left[ \| \nu_{e, i}^{\widehat \lambdad} \|_1^2\right] & \leq \frac{G_{\smp}(\eta) }{(K + 2)^2}
\end{align*}
for all $e \in E$ and $i \in e$.
Then, for $\epsilon' > 0$, we can ensure that
\begin{align*}
    \E [ \| \nu_{e, i}^{\widehat \lambdad}\|_1 ] & \leq |N_i| \epsilon' \\
    \E \sum_{i \in V, e\in N_i}  \| \nu_{\cdot, i}^{\widehat \lambdad}\|_1^2 & \leq N (\epsilon')^2
\end{align*}
in $K = \frac{\sqrt{8 \eta }G(\eta)}{\epsilon'}$ iterations.
Letting $\widehat\mu \in \LL_2$ be the projected version of $\mu^{\widehat \lambdad}$,
\begin{align*}
    \< C, \widehat \mu-\mu^*\> & \leq \| C\|_\infty  \left( 16( m + n) d \delta + \sum_{e, i}  4\| \nu^{\widehat\lambdad}_{e, i}\|_1 \right)  + \frac{n \log d + 2 m \log d}{\eta}.
\end{align*}
Taking the expectation of both sides gives us
\begin{align*}
    \E [\< C, \widehat \mu-\mu^*\>] & \leq \| C\|_\infty  \left( 16( m + n) d \E [\delta]  + 4 N\epsilon' \right) + \frac{n \log d + 2 m \log d}{\eta},
\end{align*}
where
\begin{align*}
    \E \left[ \delta \right]^2  \leq  \E [ \delta^2 ]  \leq \E  \sum_{e \in E, i \in e} \| \nu_{e, i}^{\widehat \lambdad}\|_1^2   
                     \leq N (\epsilon')^2.
\end{align*}
Then we can conclude
\begin{align*}
        \E [ \< C, \widehat \mu-\mu^*\>] & \leq   16\sqrt{N} ( m + n) d \| C\|_\infty \epsilon'  + 4N \| C\|_\infty \epsilon'  + \frac{n \log d + 2 m \log d}{\eta} \\
        & \leq 24\sqrt{N} ( m + n) d \| C\|_\infty \epsilon'  + \frac{n \log d + 2 m \log d}{\eta}.
    \end{align*}
    The last inequality uses the fact that $N = 2m$. Therefore, $\widehat \mu$ is expected $\epsilon$-optimal with $\eta$ as defined in the statement and $\epsilon' = \frac{\epsilon}{48\sqrt{N} ( m + n) d \| C\|_\infty}$. Substituting these values into $K$ and $G(\eta)$ yields the result.
\end{proof}

\section{Rounding to Integral Solutions Proofs}\label{section::rounding_guarantees}

In this section, we prove the bound on the number of iterations sufficient to recover the MAP solution using \aemp{} and rounding the output of the algorithm. We then compare with standard \emp{}.

\subsection{Approximation Error}

Let $\mathcal{V}_2$ be the set of vertices of $\LL_2$ and $\mathcal{V}_2^*$ be the set of optimal vertices with respect to $C$. Denote by $\Delta = \min_{V_1 \in \mathcal{V}_2 \backslash \mathcal{V}_2^*, V_2 \in \mathcal{V}_2^*} \langle C, V_1 - V_2 \rangle $ the suboptimality gap. Let $\mathcal{R}_1 = \max_{\mu \in \LL_2} \| \mu \|_1$, and $\mathcal{R}_H = \max_{\mu, \mu' \in \LL_2} H(\mu) - H(\mu')$. Define $\deg$ to be the maximum degree of the graph. The following holds:
\begin{theorem}[Theorem 1 of \cite{lee2019approximate}] If $\LL_2$ is tight, $|\mathcal{V}_2^*| = 1$ and $\eta \geq \frac{2\mathcal{R}_1 \log(64\mathcal{R}_1)+ 2\mathcal{R}_1 + 2\mathcal{R}_H}{\Delta}$, then $\| \mu_\eta^* - \mu^* \|_1 \leq \frac{1}{8}$ and therefore the rounded solution $\mathrm{round}(\mu_\eta^*)$ is a MAP assignment. 
\end{theorem}

\subsection{Estimation Error for Accelerated Message Passing}

To bound the estimation error, we invoke the accelerated convergence guarantees presented in the previous section. In particular, we showed that
\begin{align*}
    \E \left[\| \nu_{e, i}^{\widehat \lambdad}\|_1\right] \leq \epsilon' \quad \forall e \in E, i \in e
\end{align*}
after $K = \frac{\sqrt{4 \eta} G(\eta)}{\epsilon'}$ iterations for \aemp{}. Markov's inequality implies that with probability $1 - \delta$, $\| \nu_{e, i}^{\widehat \lambdad}\|_1 \leq \frac{2m \epsilon' }{\delta} :=\epsilon$ for all $e \in E$ and $i \in e$. From Theorem 3 of \citet{lee2019approximate}, we require
\begin{align*}
  \epsilon < O\left(d^{-2}m^{-2}\deg^{-2} \max(1, \eta \| C\|_\infty)^{-1}\right) 
\end{align*}
Furthermore, the theorem of the previous subsection implies we can set
\begin{align*}
    \eta &= \frac{ 16(m+n)(\log(m+n)+ \log(d))}{\Delta}
\end{align*}
Then, by setting
\begin{align*}
    \epsilon' \leq O\left( d^{-2} m^{-4} \delta \deg^{-2} \max(1, \|C\|_\infty/\Delta)^{-1} (\log dm)^{-1}\right)
\end{align*}
the condition is satisfied. Therefore, plugging into $\sqrt{4 \eta} G(\eta)$ yields
\begin{align*}
    \sqrt{4 \eta} G(\eta) = O\left(\frac{ dm^3 \|C\|_\infty \log dm }{\Delta}\right)
\end{align*}
which implies, with probability $1 - \delta$,
\begin{align*}
    K = O \left( \frac{d^3 m^7 \deg^2 \| C\|_\infty^2 \log^2 dm}{\delta \Delta} \right)
\end{align*}
These conditions of $\epsilon'$ and $\eta$ guarantee that the $\mathrm{round}(\mu^{\widehat \lambdad})$ is the MAP solution by invoking Theorem 3 of \citet{lee2019approximate}.

\subsection{Comparison to Standard Methods}

Using standard \emp{}, we require the same conditions be satisfied on $\epsilon'$ and $\eta$ to guarantee recover of the MAP solution.  However, the rate of convergence differs, requiring $K = \frac{L(0) - L(\lambdad^*)}{(\epsilon')^2}$ iterations, as seen previously. Note that \begin{align*}
    L(0) - L(\lambdad^*) & = O\left( \frac{m^3 \|C\|_\infty \log dm}{\Delta}\right)
\end{align*}. Note that there is no additional $d$ dependence. It holds that with probability $1 - \delta$, the MAP solution is recovered by \emp{} in at most
\begin{align*}
    K = O \left(  \frac{d^4 m^{11} \deg^4 \|C\|_\infty^3 \log^3 dm}{\delta^2 \Delta} \right)
\end{align*}
iterations. We emphasize that this iteration bound is only a sufficient condition by directly applying the technique developed in this section. We suspect it can be greatly improved.

\end{document}